\documentclass[twoside]{article}

\usepackage{Definitions}
\usepackage{amssymb}
\usepackage{amsmath}
\usepackage{color}
\usepackage{subcaption}
\usepackage{graphicx}
\usepackage[accepted]{aistats2020}
\usepackage[round]{natbib}
\renewcommand{\bibname}{References}
\renewcommand{\bibsection}{\subsubsection*{\bibname}}

\allowdisplaybreaks

\begin{document}



\twocolumn[

\aistatstitle{Elimination of All Bad Local Minima in Deep Learning}

\aistatsauthor{Kenji Kawaguchi \And Leslie Pack Kaelbling}

\aistatsaddress{MIT  \And MIT} 

]

\begin{abstract}
In this paper, we theoretically prove that adding one special neuron per output unit eliminates all suboptimal local minima of any deep neural network, for multi-class classification, binary classification, and regression with an arbitrary loss function, under practical assumptions. At every local minimum of any deep neural network with these added neurons, the set of parameters of the original neural network (without added neurons) is guaranteed to be a global minimum of the original neural network. The effects of the added neurons are proven to automatically vanish at every local minimum. Moreover, we provide a novel theoretical characterization of a failure mode of eliminating suboptimal local minima via an additional theorem and several examples. This paper also introduces a novel proof technique based on the perturbable gradient basis (PGB) necessary condition of local minima, which provides new insight into the elimination of local minima and is applicable to analyze various models and transformations of objective functions beyond the elimination of local minima. 
\end{abstract}

\section{Introduction}
Deep neural networks have achieved significant practical success in the fields of computer vision, machine learning, and artificial intelligence. However, theoretical understanding of deep neural networks is scarce relative to its empirical success. One of the major difficulties in theoretically understanding deep neural networks lies in the non-convexity and high-dimensionality of the objective functions used to train the networks. Because of the non-convexity and high-dimensionality, it is often unclear whether a deep neural network will be guaranteed to have a desired property after training, instead of becoming stuck around an arbitrarily poor local minimum. Indeed, it is NP-hard to find a global minimum of a general non-convex function \citep{murty1987some}, and of non-convex objective functions used to train certain types of neural networks \citep{blum1992training}. In the past, such theoretical concerns were considered  one of reasons to prefer classical machine learning models (with or without a kernel approach) that require only convex optimization. Given their recent empirical success, a question remains whether practical deep neural networks can be theoretically guaranteed to avoid poor local minima.

There have been numerous recent   studies that have advanced theoretical understanding regarding  the optimization  of  deep neural networks with    significant over-parameterization \citep{nguyen2017loss,nguyen2018optimization,allen2018convergence,du2018gradient,zou2018stochastic} and model simplification \citep{choromanska2015loss,kawaguchi2016deep,hardt2016identity,bartlett2019gradient,du2019width}.  For shallow networks with a single hidden layer, there have been many positive results, yet often with strong assumptions, for example, requiring the use of significant over-parameterization, simplification, or Gaussian inputs \citep{andoni2014learning,sedghi2014provable,soltanolkotabi2017learning,brutzkus2017globally,ge2017learning,soudry2017exponentially,goel2017learning,zhong2017recovery,li2017convergence,du2018power}.   

Instead of using these strong assumptions, adding one neuron to a neural network with a single output unit was recently shown to be capable of eliminating all suboptimal local minima (i.e., all local minima that are not global minima) for binary classification  with a special type of smoothed hinge loss function \citep{liang2018adding}. However, the restriction to  a neural network with a single output unit and  a special loss function for binary classification makes it inapplicable to many common and important deep learning problems. Both removing the restriction to networks with a single output unit and removing restrictions on loss functions are seen as important open problems in different but related theoretical work on local minima of neural networks~\citep{shamir2018resnets, laurent2018deep}.

In this paper, we state and prove a novel and significantly stronger theorem:  adding one neuron per output unit can eliminate all suboptimal local minima of any deep neural network with an arbitrary loss function for multi-class classification, binary classification, and regression.  This paper also introduces a novel proof technique based on the perturbable gradient basis (PGB) condition, which provides new insight into the elimination of local minima and can be used to analyze new models and transformations of objective functions beyond the elimination of local minima. This paper analyzes the problem in the regime without  significant over-parameterization or model simplification (except adding the few extra neurons).

While analyzing the properties of local minima in this regime with no strong assumption is a potentially  important step in theory, it does not immediately guarantee the efficient solution of general neural network optimization.  We explain this phenomenon in another key contribution of this paper, which is a novel characterization of a   failure mode of the removal of bad local optima, in terms of its effect on gradient-based  optimization methods.

\section{Elimination of local minima} \label{sec:main_results}

The optimization problem for the elimination of local minima is defined in Section \ref{sec:problem_description}. Our theoretical results on the elimination of local minima are presented in Section \ref{sec:main_results_arbitrary_datasets} for arbitrary datasets, and in Section \ref{sec:main_results_realizable_datasets} for realizable datasets.
We discuss these results in terms of non-vacuousness, consistency, and the effect of over-parameterization in Section \ref{sec:over-para-effect}.
\subsection{Problem description}  \label{sec:problem_description}

Let $x \in \Xcal \subseteq \RR^{d_x}$ and $y \in \Ycal \subseteq\RR^{dy}$ be an input vector and a target  vector, respectively. Define $((x_i,y_i))_{i=1}^m$ as a training dataset of size $m$.
Given an input $x$ and parameter $\theta$, let $f(x;\theta) \in \RR^{d_y}$ be the pre-activation output
of the last layer of any arbitrary deep neural network with any structure (e.g., any convolutional neural network with any depth and any width, with or without skip connections). That is, there is no assumption with regard to $f$ except that $f(x;\theta) \in \RR^{d_y}$. 

We  consider the following standard objective function $L$ to train an arbitrary neural network $f$:  
$$
L(\theta) =\frac{1}{m} \sum_{i=1}^m \ell(f(x_{i};\theta),y_{i}),
$$
where $\ell:\mathbb{R}^{d_y} \times\Ycal \rightarrow \RR$ is an arbitrary loss criterion such as cross entropy loss, smoothed hinge loss, or squared loss. 

We then define an auxiliary objective function $\tilde L$:
$$
\tilde L(\tilde \theta)=\frac{1}{m} \sum_{i=1}^m \ell( f(x_{i};\theta) + g(x_{i};a,b,W),y_{i}) +\lambda \|a\|^2_2,
$$  
where $\lambda>0$, $\tilde \theta=(\theta,a,b,W)$, $a,b \in \RR^{d_y}$,  $W=\begin{bmatrix}w_{1} & w_{2} & \cdots & w_{d_y} \end{bmatrix} \in \RR^{d_x \times d_y}$ with $w_k \in \RR^{d_x}$, and 
$
g(x;a,b,W)_{k} =a_{k} \exp(w_k\T x +b_k) 
$
for all $k \in \{1,\dots,d_y\}$. 

We also define a modified neural network $\tilde f$ as 
$$
\tilde f (x; \tilde \theta)=f(x;\theta) + g(x;a,b,W),
$$
which is equivalent to adding one neuron $g(x;a,b,W)_{k}$ per each output unit $f(x;\theta)_k$ of the original neural network. Because $\tilde L(\tilde \theta)=\frac{1}{m} \sum_{i=1}^m \ell(\tilde f(x_{i};\tilde \theta),y_{i})+\lambda \|a\|^2_2$, the auxiliary objective function $\tilde L$ is the standard objective function $L$ with the modified neural network $\tilde f$ and a regularizer on $a$.

\subsection{Result for arbitrary datasets} \label{sec:main_results_arbitrary_datasets}
\vspace{-5pt}
Under only a mild assumption (Assumption \ref{assump:loss}), Theorem \ref{thm:main} states that at every local minimum $(\theta,a,b,W)$ of the modified objective function $\tilde L$, the parameter vector $\theta$ achieves a global minimum of the original objective function $L$, and the modified neural network $\tilde f$ automatically becomes the original neural network $f$. The proof of Theorem \ref{thm:main} is presented in Section \ref{sec:proof_new} and Appendix \ref{sec:proof}. 

\begin{assumption} \label{assump:loss}
\emph{(Use of common loss criteria)} For any $i \in \{1,\dots,m\}$, the function $\ell_{y_i}:q \mapsto\ell(q, y_{i})\in \RR_{\ge 0}$ is differentiable and convex (e.g., the squared loss, cross entropy loss, and polynomial hinge loss satisfy this assumption).  
\end{assumption}

\begin{theorem} \label{thm:main}
Let Assumption \ref{assump:loss} hold. Then, at every local minimum $(\theta,a,b,W)$ of $\tilde L$,  the following statements hold:
\vspace{-3pt}
 \begin{enumerate}[label=(\roman*)]
\item \vspace{-3pt}
 $\theta$ is a global minimum of $L$, and   
\item \vspace{-3pt}
$\tilde f (x;\theta,a,b,W) = f(x;\theta)$ for all $x \in \RR^{d_x}$, and $\tilde L(\theta,a,b,W) = L(\theta)$.
\end{enumerate}
\end{theorem} \vspace{-3pt}

Assumption \ref{assump:loss} is satisfied by simply using a common loss criterion, including the squared loss  $\ell(q,y)= \|q-y\|_2^2$ or   $\ell(q,y) = (1- yq)^2$ (the latter  with $d_y=1$),
cross entropy loss  $\ell(q,y)=-\sum_{k=1}^{d_y}y_k \log \frac{\exp(q_k)}{\sum_{k'}\exp(q_{k'})}$,
or smoothed hinge loss  $\ell(q,y)=(\max\{0,1-y q\})^p$ with $p \ge 2$ (the hinge loss with $d_y=1$).
Although the objective function $L:\theta \mapsto L(\theta)$ used to train a neural network is non-convex in $\theta$, the loss criterion $\ell_{y_i}:q  \mapsto\ell(q, y_{i})$ is usually convex in $q$. Therefore, Theorem \ref{thm:main} is directly applicable to most common deep learning tasks in practice. %This means that, in practice, one can eliminate all suboptimal local minima by simply adding one neuron per output unit. Furthermore, the added neurons automatically vanish  and thus do not affect the output of the neural network at every local minimum. 

\subsection{Result for realizable datasets} \label{sec:main_results_realizable_datasets}

Theorem \ref{thm:realizable_data} makes a statement similar to Theorem \ref{thm:main} under a weaker assumption on the loss criterion (Assumption \ref{assump:loss2}) but with an additional assumption on the training dataset  (Assumption \ref{assump:realizable}). The proof of Theorem \ref{thm:realizable_data} is presented in Appendix \ref{sec:proof}.

\begin{assumption} \label{assump:loss2}
\emph{(On the loss)} For any $i\in\{1,\dots,m\}$, the function $\ell_{y_i}:q  \mapsto\ell(q, y_{i})$ is differentiable, and $q \in \RR^{d_y}$ is a global minimum of $\ell_{y_i}$ if $\nabla \ell_{y_i}(q)=0$.
\end{assumption}

\begin{assumption} \label{assump:realizable}
\emph{(On the label consistency)} There exists a function $f^*$ such that $f^*(x_i) = y_i$ for all $i \in \{1,\dots,m\}$.
\end{assumption}
 
\begin{theorem} \label{thm:realizable_data}
Let Assumptions \ref{assump:loss2} and \ref{assump:realizable} hold. Then, at every local minimum $(\theta,a,b,W)$ of $\tilde L$,
 the following statements hold:
\vspace{-8pt}
\begin{enumerate}[label=(\roman*)]
\item 
 $\theta$ is a global minimum of $L$,    
\item \vspace{-3pt}

$\tilde f (x;\theta,a,b,W) = f(x;\theta)$ for all $x \in \RR^{d_{x}}$, and $\tilde L(\theta,a,b,W) = L(\theta)$, and
\item \vspace{-3pt}

 $f(x_i;\theta)$ is  a global minimum of $\ell_{y_i}:q  \mapsto\ell(q, y_{i})$ for all $i\in\{1,\dots,m\}$.    
\end{enumerate}
\end{theorem}

        Assumption \ref{assump:loss2}
is weaker than Assumption \ref{assump:loss} in the sense that the former is implied by the latter but not vice versa. However, as discussed above, Assumption \ref{assump:loss} already accommodates most common loss criteria. Assumption \ref{assump:realizable} is automatically satisfied if a target $y$ given an input $x$ is not random, but the non-randomness is not necessary to satisfy Assumption \ref{assump:realizable}. Even if the targets are generated at random, as long as all $x_1,x_2,\dots,x_m$ are distinct (i.e., $x_i\neq x_j$ for all $i\neq j$), Assumption \ref{assump:realizable} is satisfied.

Therefore, although Theorem \ref{thm:realizable_data} might be less applicable in practice when compared to Theorem \ref{thm:main}, 
 Theorem \ref{thm:realizable_data} can still be applied to many common deep learning tasks with the additional guarantee, as stated in Theorem \ref{thm:realizable_data} (iii). By using an appropriate loss criterion for classification, Theorem \ref{thm:realizable_data} (iii)  implies that the trained neural network $f(\cdot;\theta)$ at every local minimum correctly classifies all training data points.

\subsection{Non-vacuousness, consistency, and  effect of over-parameterization} \label{sec:over-para-effect}
Theorems \ref{thm:main} and \ref{thm:realizable_data}  are both  non-vacuous and consistent with pathological cases. For the consistency, Theorems \ref{thm:main} and \ref{thm:realizable_data}  (vacuously) hold true if there is no local minimum of $\tilde L$, for example, with a pathological case of $\ell(q, y_{i})=q-y_{i}$. 

For non-vacuousness, there exists a local minimum of $\tilde L$ if there exists a global minimum $\theta$ of $L$ such that $f(x_i;\theta)$ achieves a global minimum for each $f(x_{i};\theta) \mapsto\ell(f(x_{i};\theta),y_i)$ for $i \in \{1,\dots,m\}$ (this is because, given such a $\theta$, any point with $a=0$ is a local minimum of $\tilde L$). Therefore, the existence of local minimum for $\tilde L$ can be guaranteed by the \textit{weak} degree of over-parameterization that ensures the exitance of  a global minimum $\theta$ for each $f(x_{i};\theta) \mapsto\ell(f(x_{i};\theta),y_i)$ \textit{only for a given training dataset} (rather than for all datasets). 

This is in contrast to the previous papers that require \textit{significant}     over-parameterization to ensure interpolation of  all datasets and to make the corresponding neural tangent kernel to be approximately unchanged during training \citep{nguyen2017loss,nguyen2018optimization,allen2018convergence,du2018gradient,zou2018stochastic}. Our paper does not require those and allow the neural tangent kernel to significantly change during training. Because of this difference, our paper only needs  $\tilde \Omega(1)$ parameters,  whereas the state-of-the-art previous paper requires $\tilde \Omega(H^{12}n^8)$ parameters \citep{zou2019improved}  or $\tilde \Omega(n)$ parameters \citep{kawaguchiAllerton2019},  .    

Because a local minimum does not need to be a \textit{strict} local minimum (i.e., a local minimum with a strictly less  value than others in a neighborhood), there are many other cases where there exists a local minimum of $\tilde L$: e.g., Example \ref{example:3-2} in Section \ref{sec:analytical_example} also illustrates a situation where there  exists a local minimum of $\tilde L$ without the above condition of the existence of sample-wise global minimum of $L$ or weak over-parameterization. 

\section{PGB necessary condition beyond elimination of  local minima}
In this section, we  introduce a more general  result beyond  elimination of local minima. Namely, we prove the \textit{perturbable gradient basis (PGB) necessary condition} \textit{of local minima}, which directly implies the elimination result as a special case. 
Beyond the specific transformation of the objective function $\tilde L$, the PGB necessary condition of local minima can be applied to  other objective functions with various transformations and models.

\begin{theorem} \label{thm:pgb}
\emph{(PGB necessary condition of local minima)}. Define the objective function $Q$ by 
\begin{align} \label{eq:pgb_1}
Q(z) =\frac{1}{m} \sum_{i=1}^m Q_{i}(\phi_{i}(z))+R_{i}(\varphi_i(z)) 
\end{align} 
where   for all $i \in \{1,\dots,m\}$, the functions $Q_{i}:q \in \RR^{d_\phi}\mapsto Q_{i}(q)\in \RR_{\ge 0}$ and  $R_{i}: q\in \RR^{d_\varphi} \mapsto R_{i}(q)\in \RR_{\ge 0}$ are differentiable and convex, and  $\phi_{i}$ and $\varphi_i$ are differentiable. Assume that there exists a function $h^{}: \RR^{d_z} \rightarrow \RR^{d_z}$  and a real number $\rho\neq0$ such that for all $i \in \{1,\dots,m\}$ and all $z \in \RR^{d_z}$, $\phi_{i}(z)=\sum_{k=1}^{d_z} h^{}(z)_k \partial_{k} \phi_{i} (z)$ and $\varphi_{i}(z)=\rho\sum_{k=1}^{d_z} h^{}(z)_k \partial_{k} \varphi_{i} (z)$. Then, for any local minimum $z\in \RR^{d_z} $ of $Q$, the following holds: there exists $\epsilon_{0} > 0$ such that for any $\epsilon \in [0, \epsilon_{0})$,
\begin{align*} 
\scalebox{0.95}{$\displaystyle Q(z)\le \inf_{\substack{S \finsubseteq \Vcal[z,\epsilon], \\ \alpha \in \RR^{d_z \times |S|} }} \tilde Q_{\epsilon,z}(\alpha,S)+ \frac{\rho-1}{\rho m}  \sum_{i=1}^m \partial R_{i}(\varphi_i(z))\varphi_i(z),$}
\end{align*}
where 
$$\tilde Q_{\epsilon,z}(\alpha,S)=\frac{1}{m}\sum_{i=1}^m  Q_{i}(\phi_{i}^{z}(\alpha, \epsilon,S))+R_{i}(\varphi_i^{z}(\alpha, \epsilon,S)),$$ 
$$\phi_{i}^z(\alpha, \epsilon,S)=\sum_{k=1}^{d_z} \sum_{j=1}^{|S|} \alpha_{k,j} \partial _{k}\phi_{i}(z+\epsilon S_{j}),$$ and 
$$\varphi_i^z(\alpha, \epsilon,S)=\sum_{k=1}^{d_z} \sum_{j=1}^{|S|} \alpha_{k,j} \partial _{k}\varphi_i(z+\epsilon S_{j}).$$ Here, $S \finsubseteq S'$  denotes a finite subset $S$ of a set  $S'$ and
  $\Vcal[z,\epsilon]$ is the set of all vectors $v \in \RR^{d_z}$ such that $\|v\|_2\le1$, $\phi_{i}(z+ \epsilon v)=\phi_{i}(z)$, and $\varphi_{i}(z+ \epsilon v)=\varphi_{i}(z)$ for all $i \in \{1,\dots,m\}$. Furthermore, if $\rho=1$, this statement holds with equality as $Q(z)
 = \inf_{\substack{S \finsubseteq \Vcal[z,\epsilon],\alpha \in \RR^{d_z \times |S|} }} \tilde Q_{\epsilon,z}(\alpha,S)$.
\end{theorem}

The PGB necessary condition of local minima
states that whether a given objective $Q$ is non-convex or convex, if $z$ is a \textit{local} minimum of the original (potentially non-convex) objective  $Q$, then $z$ must achieve the \textit{global} minimum value of the transformed objective  $\tilde Q_{\epsilon,z}(\alpha,S)$ for any sufficiently small $\epsilon$. Here, the transformed objective $\tilde Q_{\epsilon,z}$ is the original objective $Q$ except that the original functions  $\phi_{i}(z)$ and $\varphi_i(z)$ are replaced by perturbable gradient basis (PGB) functions $\phi_{i}^z(\alpha, \epsilon,S)$ and $\varphi_i^{z}(\alpha, \epsilon,S)$. In other words, all local minima with the original functions  $\phi_{i}(z)$ and $\varphi_i(z)$ achieve the global minimum values of PGB functions  $\phi_{i}^z(\alpha, \epsilon,S)$ and $\varphi_i^{z}(\alpha, \epsilon,S)$. 

Here, the original  objective  $Q$ is a non-convex function in general because both $\phi_{i}$ and $\varphi_i$ can be non-convex functions. If  $\phi_{i}$ and $\varphi_i$ are both linear functions, then  we have  $\phi_{i}(\alpha)=\phi_{i}^z(\alpha, 0,\emptyset)$ and $\varphi_i(\alpha)=\varphi_{i}^z(\alpha, 0,\emptyset)$ (for all $\alpha$ and $z$) and hence the PGB necessary condition of local minima recovers  the following known statement: every local minimum of original $Q$ is a global minimum of original $Q$.

The PGB condition of local minima can be also understood based on  the following geometric interpretation. For the geometric interpretation, we consider two spaces -- the parameter space $\RR^{d_z}$ and the output space $\RR^{m(d_{\phi}+d_\varphi)}$ -- and the  map from the parameter space to the output space, which is defined by  
$
\Phi: z\in \RR^{d_z} \mapsto  (\phi_{1}(z)\T, \varphi_1(z)\T, \dots, \phi_{m}(z)\T, \varphi_m(z)\T)\T \in\RR^{m(d_{\phi}+d_\varphi)}.   
$
Then, in the output space, we can intuitively consider the ``tangent'' space $T_{\Phi(z)}=\Span(\{\partial_1\Phi(z), \dots, \partial_{d_z} \Phi(z)\})+\{\Phi(z)\}$, where the sum of the two sets represents the Minkowski sum of the sets.\footnote{In special cases (e.g., when $\rank(\partial\Phi(z))$ is constant in a neighborhood of $z$), $T_{\Phi(z)}$ is indeed a tangent space of a local manifold embedded in the output space. In general, $T_{\Phi(z)}$ and $\tilde T_{\Phi(z)}$ are the affine subspaces of the output space $\RR^{m(d_{\phi}+d_\varphi)}$.} Then, given a $\epsilon$ ($\le \epsilon_0$), the span of the set of all vectors of the ``tangent'' spaces $T_{\Phi(z+\epsilon v)}$ at all perturbed points $z+\epsilon v$, defined by $
\tilde T_{\Phi(z+\epsilon v)} = \Span(\{\mathbf{f}\in \RR^{m(d_{\phi}+d_\varphi)}: (\exists v\in {\cal V}[\theta, \epsilon])[\mathbf{f} \in T_{\Phi(z+\epsilon v)}]\}),
$ is  exactly equal to the space of  the outputs of  the PGB functions. 
 
Therefore, from the geometric viewpoint, the PGB necessary condition  of local minima states that the output $\Phi(z)$  at any local minimum $z$ is globally optimal in the  span of the ``tangent'' spaces $\tilde T_{\Phi(z)}$. The PGB condition  of local minima translates the \textit{local} optimality in the \textit{parameter} space $\RR^{d_z}$ into the \textit{global} optimality in the span of the ``tangent'' spaces in the \textit{output space} $\RR^{m(d_{\phi}+d_\varphi)}$.

The PGB necessary condition of local minima is an  extension  of theorem 2 in a previous study \citep{kawaguchi2019every} to the problem with a regularization term and a general transformation  of the objective function. Accordingly, beyond the elimination of local minima, the PGB necessary condition of local minima  entails   the previously proven statements of no bad local minima for  deep linear neural networks \citep{laurent2018deep} and deep nonlinear residual neural networks \citep{kawaguchi2019depth} (since the PGB necessary condition of local minima is strictly more general than theorem 2  by \citealt{kawaguchi2019every}, which was shown to entail  those statements).     

With our extension, the PGB\ necessary condition of local minima can be now used to study the effects of  various transformations of objective functions and regularization terms, including the elimination of local minima, as shown in the next section.

\vspace{-5pt}
\section{Application of PGB necessary condition for elimination of  local minima} \label{sec:proof_new}
\vspace{-5pt}

In this section, we introduce a novel and concise proof for Theorem \ref{thm:main} based on the PGB necessary condition of local minima, which shows that   all suboptimal local minima can be eliminated because the global minimum value of the PGB model is indeed the global minimum value of the original model. From the geometric viewpoint, this is  because the PGB model is shown to be expressive in that  the span of  the ``tangent'' spaces $T_{\Phi(z+\epsilon v)}$ contains the output space $\RR^{md_{\phi}}\times \{\mathbf 0\}$ where $\mathbf 0 \in \RR^{md_\varphi}$. In Appendix \ref{sec:proof}, we  also provide an alternative   proof of Theorem \ref{thm:main} without the PGB necessary condition, which is intended to be more detailed with elementary facts but   less elegant than the  following proof via the PGB\ necessary condition.

\begin{proof}[Proof of Theorem \ref{thm:main}]        
Let $\theta$ be fixed. Let $(a,b,W)$ be a local minimum of $\tilde L|_{\theta}(a,b,W):=\tilde L(\theta,a,b,W)$. Let $\bar \theta \in \RR^{d_{\bar \theta}}$ be the vector containing $(a,b,W)$, defined by $\bar \theta=(a\T,b\T,\vect(W)\T)\T$. We apply the PGB\ necessary condition of local minima by setting $Q(\bar \theta)=\tilde L|_{\theta}(a,b,W )$ with $\phi_{i}(\bar \theta)=g(x_{i};a,b,W)$, $\varphi_{i}(\bar \theta)=(a_1^{2},\dots,a_{d_y}^2)\T$, $Q_i(q)=\ell( f(x_{i};\theta) + q,y_{i})$, and $R_i(q)=\lambda \sum_{j=1}^{d_y}q_j$.
Then, for all $i \in \{1,\dots,m\}$, the functions $Q_{i}$ and  $R_{i}$ are differentiable and convex, and the functions $\phi_{i}$ and $\varphi_i$ are differentiable.\ Furthermore, we can rewrite    $\phi_{i}(\bar \theta)=\sum_{k=1}^{d_{y}} a_{k} \partial _{a_{k}}\phi_{i}(\bar \theta)$ and $\varphi_{i}(\bar \theta)=\rho\sum_{k=1}^{d_{y}} a_{k} \partial _{a_{k}}\varphi_{i}(\bar \theta)$ with $\rho=1/2$ for all $i \in \{1,\dots,m\}$ and all $\bar \theta \in \RR^{d_{\bar \theta}}$. These satisfy the assumptions of the PGB necessary condition of local minima of $Q(\bar \theta)=\tilde L|_{\theta}(a,b,W)$. 

From  the PGB necessary condition of local minima of $Q(\bar \theta)=\tilde L|_{\theta}(a, b,W)$,   there exists $\epsilon_{0} > 0$ such that for any $\epsilon \in [0, \epsilon_{0})$,
\begin{align} \label{eq:new_proof_1}\ 
&Q(\bar \theta)- \frac{\rho-1}{\rho m}  \sum_{i=1}^m \partial R_{i}(\varphi_i(\bar \theta))\varphi_i(\bar \theta) \\ \nonumber &\le \inf_{\substack{S \finsubseteq \Vcal[\bar \theta,\epsilon], \\ \alpha \in \RR^{d_z \times |S|} }} \tilde Q_{\epsilon,\bar \theta}(\alpha,S)\le \inf_{\substack{S \finsubseteq \bar \Vcal[\bar \theta,\epsilon], \\ \alpha \in \RR^{d_z \times |S|} }} \tilde Q_{\epsilon,\bar \theta}(\alpha,S),
\end{align}
where 
  $\Vcal[\bar \theta,\epsilon]$ is the set of all vectors $v \in \RR^{d_{\bar \theta}}$ such that $\|v\|_2\le1$, $\phi_{i}(\bar \theta+ \epsilon v)=\phi_{i}(\bar \theta)$, and $\varphi_{i}(\bar \theta+ \epsilon v)=\varphi_{i}(\bar \theta)$ for all $i \in \{1,\dots,m\}$. Then, the subset    $\bar \Vcal[\bar \theta,\epsilon] \subset\Vcal[\bar \theta,\epsilon] $ is defined by $\bar \Vcal[\bar \theta,\epsilon]=\{(a\T,b\T,\vect(W)\T)\T:(a\T,b\T,\vect(W)\T)\T\in\Vcal[\bar \theta,\epsilon],a= 0\}$. 

Since $(a,b,W)$ is a local minimum and hence the partial derivatives with respect to $(a,b)$ are zeros, we have that for all $k \in \{1,2,\dots,d_y\}$, $a_k \frac{\partial \tilde L(\theta,a,b,W)}{\partial a_{k}}= \frac{1}{m}\sum_{i=1}^m (\nabla \ell_{y_i}( f(x_{i};\theta) + g(x_{i};a,b,W)))_{k}a_k\exp(w_k\T x +b_k)+2\lambda a_{k}=\frac{\partial \tilde L(\theta,a,b,W)}{\partial b_{k}}+2\lambda a_{k}^2=2\lambda a_{k}^2=0$, which implies that $a_{k}=0$ for all $k \in \{1,2,\dots,d_y\}$, since $ 2\lambda \neq 0$.
Since $a=0$, it proves statement (ii), and we have 
$L(\theta) =\tilde L|_{\theta}(a,b,W)$, $(\partial _{a_k}\varphi_{i}(\bar \theta+\epsilon v))_k=2a_k=0$ for all $v \in\bar \Vcal[\bar \theta,\epsilon] $, and   $\frac{\rho-1}{\rho m}  \sum_{i=1}^m \partial R_{i}(\varphi_i(\bar \theta))\varphi_i(\bar \theta)=0$. Since $L(\theta) =\tilde L|_{\theta}(a,b,W)=Q(\bar \theta)=Q(\bar \theta)-\frac{\rho-1}{\rho m}  \sum_{i=1}^m \partial R_{i}(\varphi_i(\bar \theta))\varphi_i(\bar \theta)$ with   \eqref{eq:new_proof_1} and $\partial _{a_{k}}\varphi_{i}(\bar \theta+\epsilon v)=0$ (for all $v \in\bar \Vcal[\bar \theta,\epsilon] $), we have that for any  $\theta'$, 
\begin{align} \label{eq:proof:elimination_1} 
&L(\theta) - L(\theta') 
\\ \nonumber & \le \inf_{\substack{S \finsubseteq \bar \Vcal[\bar \theta,\epsilon], \\ \alpha \in \RR^{d_z \times |S|} }} \tilde Q_{\epsilon,\bar \theta}(\alpha,S)-L(\theta')
\\\nonumber & =  \inf_{\substack{S \finsubseteq \bar \Vcal[\bar \theta,\epsilon], \\ \alpha \in \RR^{d_{\bar \theta} \times |S|} }}\frac{1}{m} \sum_{i=1}^{m} Q_{i}(\phi_{i}^{\bar \theta}(\alpha, \epsilon,S))-\ell(f(x_i;\theta'),y_{i})
\\ \nonumber &\le0,
\end{align}  
where the last inequality  is to be shown to hold in the following. 

Since $\phi_{i}^{\bar \theta}(\alpha, \epsilon,S)$ can differ for different indexes $i$ only through different inputs $x_{i}$,  we can rewrite $\phi_{x_{i}}^{\bar \theta}(\alpha, \epsilon,S)=\phi_{i}^{\bar \theta}(\alpha, \epsilon,S)$. We then rearrange the sum in the last line of \eqref{eq:proof:elimination_1}: 
\begin{align*}
&\frac{1}{m} \sum_{i=1}^{m} Q_{i}(\phi_{x_{i}}^{\bar \theta}(\alpha, \epsilon,S))-\ell(f(x_i;\theta'),y_{i})
\\ & =\frac{1}{m} \sum_{j=1}^{m'} \sum_{i\in \Ical_j}   Q_{i}(\phi_{\bar x_{j}}^{\bar \theta}(\alpha, \epsilon,S))-\ell(f(\bar x_j;\theta'),y_{i}).
\end{align*}
where  $\{\mathcal{I}_1,\dots,\mathcal{I}_{m'}\}$ is a partition\footnote{That is,  $\mathcal{I}_1 \cup \cdots \cup \mathcal{I}_{m'}=\{1,\dots,m\}$, $\mathcal{I}_j \cap \mathcal{I}_{j'} = \emptyset$ for all $j \neq j'$, and $\mathcal{I}_j \neq \emptyset$ for all $j \in \{1,\dots,m'\}$.} of the set $\{1,\dots,m\}$ such that for any $x \in \mathcal{I}_j$ and $x' \in \mathcal{I}_{j'}$, $x=x'$ if $j=j'$, and $x\neq x'$ if $j\neq j'$. Here, we write  $\bar x_{j}:=x$ with a representative $x \in \mathcal{I}_j$.

Let $ S_{t} \in \RR^{d_{\bar \theta}} $ be the vector containing  $(\hat a^{(t)},\hat b^{(t)}, \hat W^{(t)})$. Since $a=0$, we have $\phi_{i}(\bar \theta+ \epsilon S_{t})=\phi_{i}(\bar \theta)$ for all   vectors $ S_{t}$ containing any $(\hat a^{(t)},\hat b^{(t)}, \hat W^{(t)})$ with $\hat a^{(t)}=0$.  In other words,  for any finite $|S|$ and any   $((\hat b^{(t)}, \hat W^{(t)}))_{t=1}^{|S|}$ with  $\|S_{t}\|_2\le 1$, there exists $S \finsubseteq \bar \Vcal[\bar \theta,\epsilon]$ such that $\phi_{\bar x_{j}}^{\bar \theta}(\alpha, \epsilon,S)_k=\exp(w_{k}\T  \bar x_j +b_k ) \sum_{t=1}^{|S|} \alpha_{k}^{(t)}  \exp( \epsilon( (\hat w_{k}^{(t)})\T \bar x_j + \hat b_{k}^{(t)}) )$ (by letting  $\hat a^{(t)}=0$ for all $t$). Therefore, given $m'$ distinct input points  $\bar x_1, \dots, \bar x_{m'}$,  fix  $((\hat b^{(t)}, \hat W^{(t)}))_{t=1}^{|S|}$  such that the rank of the matrix $M\in \RR^{m' \times |S|}$ with entries $M_{j,t}=\exp( (\hat w_{k}^{(t)})\T \bar x_j + \hat b_{k}^{(t)})$ is $m'$ with a  sufficiently large finite $|S|$. Then, by letting $M_{j,t}(\epsilon)\in \RR^{m' \times |S|}$ be the matrix with entries $M_{j,t}(\epsilon)=\exp(\epsilon( (\hat w_{k}^{(t)})\T \bar x_j + \hat b_{k}^{(t)}))$, the function $\psi(\epsilon)= \det(M_{j,t}(\epsilon) M_{j,t}(\epsilon)\T)$ is an analytic  function of $\epsilon$. Since $\psi$ is analytic  and  $\psi(1)\neq 0$, either the zeros of $\psi$ are isolated or  $\psi(\epsilon)\neq 0$ for all $\epsilon$. In both cases, there exists  $\epsilon \in [0, \epsilon_{0})$ such that $\psi(\epsilon)\neq0$. Fix a $\epsilon \in [0, \epsilon_{0})$ with $\psi(\epsilon)\neq 0$. Then, since $M_{j,t}(\epsilon)$ has rank $m'$, for any  $\theta'$, there exists $\alpha$ such that for all $j \in \{1,\dots,m'\}$,
$$
\phi_{\bar x_{j}}^{\bar \theta}(\alpha, \epsilon,S) =f(\bar x_j;\theta')-f(\bar x_j;\theta).  
$$            

Therefore, the last inequality in \eqref{eq:proof:elimination_1} holds, and hence for any $\theta'$, $L(\theta) \le L(\theta')$, which proves statement (i).     
\end{proof}

\section{Failure mode of eliminating suboptimal local minima}
\vspace{-7pt}

Our theoretical results in the previous sections have shown that,  for a wide range of deep learning tasks, all suboptimal local minima can be removed by  adding one neuron per output unit. This might be surprising given the fact that dealing with  suboptimal local minima in general is known to be challenging in theory.

However, for the worst case scenario, the following theorem illuminates a novel failure mode  for the elimination of  suboptimal local minima. Theorem \ref{thm:limitation} holds true for  the previous results \citep{liang2018adding} (as we discuss further in Section \ref{sec:background}) and hence  provides a novel failure mode for elimination of local minima in general, apart from our Theorems \ref{thm:main} and \ref{thm:realizable_data}. Our result in this paper is the first result that points out this type of  failure mode for elimination of local minima. The proof of Theorem \ref{thm:limitation} is presented in Appendix \ref{sec:app:proof}. 
     
\begin{theorem} \label{thm:limitation}
Let Assumption \ref{assump:loss} hold, or let Assumptions \ref{assump:loss2} and \ref{assump:realizable} hold. Then, for any $\theta$, if $\theta$ is not a global minimum of $L$, there is no local minimum $(a,b,W)\in \RR^{d_y} \times \RR^{d_y} \times \RR^{d_x \times d_y}$ of $\tilde L|_{\theta}(a,b,W):=\tilde L(\theta,a,b,W)$. Furthermore, there exists a tuple $(\ell,f, \{(x_i,y_i)\}_{i=1}^m)$  and a suboptimal stationary point  $\theta$ of $L$ such that $\frac{\partial \tilde L(\theta,a,b,W)}{\partial \theta}=0$ for all $(a,b,W)\in \RR^{d_y} \times \RR^{d_y} \times \RR^{d_x \times d_y}$. 
\end{theorem}
Therefore, on the one hand, Theorems \ref{thm:main} and \ref{thm:realizable_data} state that if an algorithm   can find a local minimum of $\tilde L$, then it can find a global minimum of $L$ (via a local minimum of $\tilde L$). On the other hand, Theorem \ref{thm:limitation}  suggests that if an algorithm  moves toward a local minimum of $\tilde L$  by simply following (negative) gradient directions, then either it moves toward a global minimum $\theta$ of $L$  or the norm of $(a,b,W)$ approaches infinity. By monitoring the norm of $(a,b,W)$, we can detect this failure mode. This suggests a hybrid approach of local and global optimization algorithms beyond a pure local gradient-based method,  with a mechanism to monitor the increase in the norm of $(a,b,W)$.

\section{Examples} \label{sec:app:experiment}
\vspace{-7pt}

The previous sections show that one can eliminate suboptimal local minima, and there remains a detectable  failure mode for gradient-based optimization methods after elimination. In this section, we provide numerical and analytical examples to illustrate these phenomena. Through these examples, we show that using $\tilde L$ instead of $L$ can help training deep neural networks in `good' cases while it may not help in `bad' cases.

\vspace{-5pt}
\subsection{Numerical examples}
\vspace{-5pt}

\begin{figure*}[!]
\center
\begin{subfigure}[b]{0.3\textwidth}
  \includegraphics[width=\textwidth]{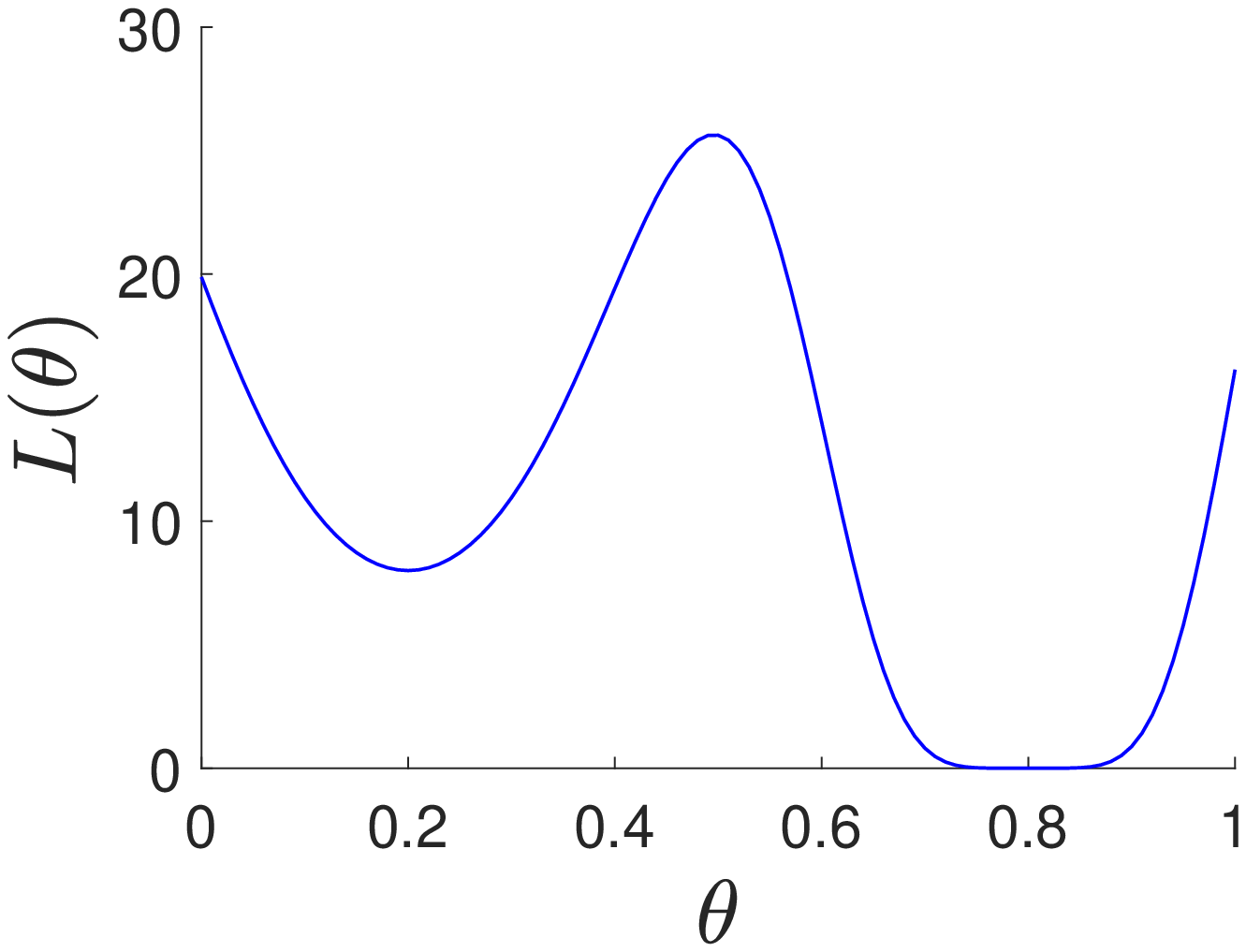} 
  \caption{objective function $L$} 
\end{subfigure}
\begin{subfigure}[b]{0.33\textwidth}
  \includegraphics[width=\textwidth]{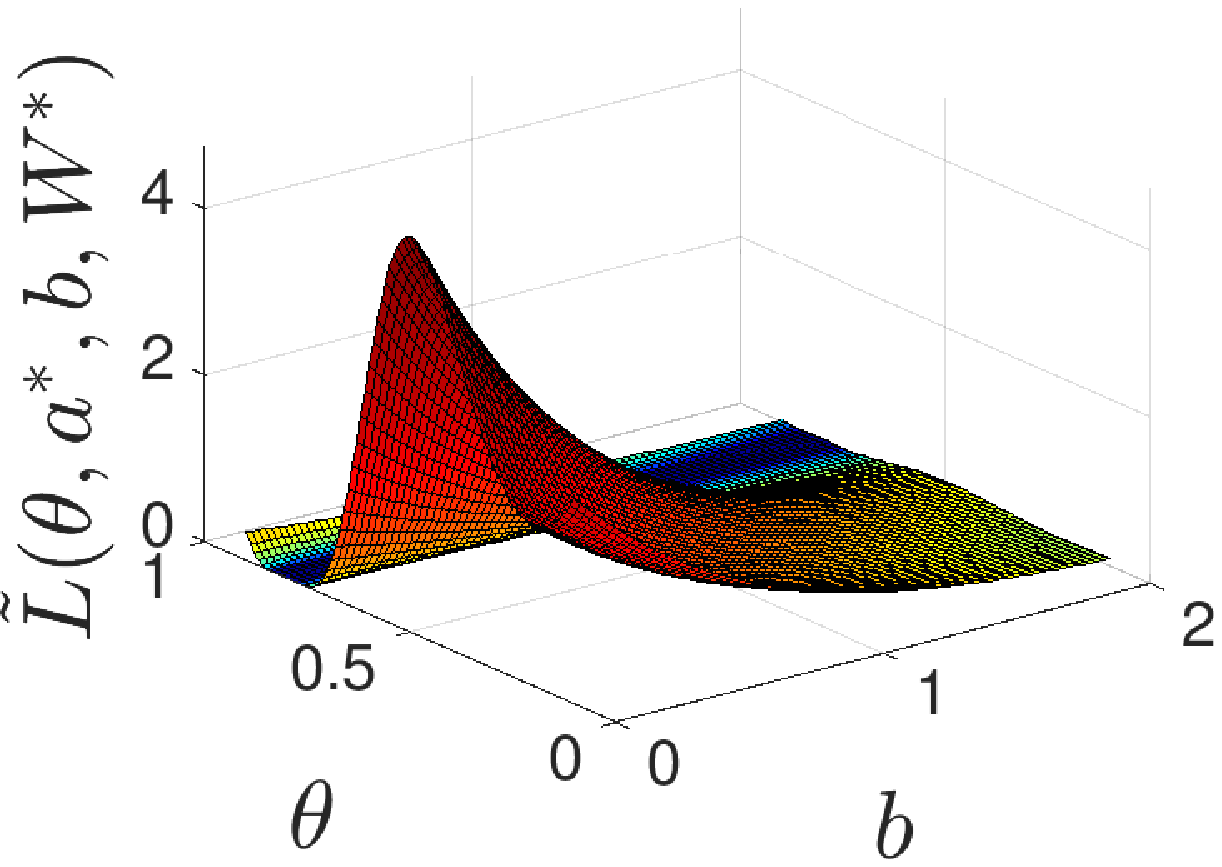}
  \caption{modified function $\tilde L$}
\end{subfigure} 
\begin{subfigure}[b]{0.35\textwidth}
  \includegraphics[width=\textwidth]{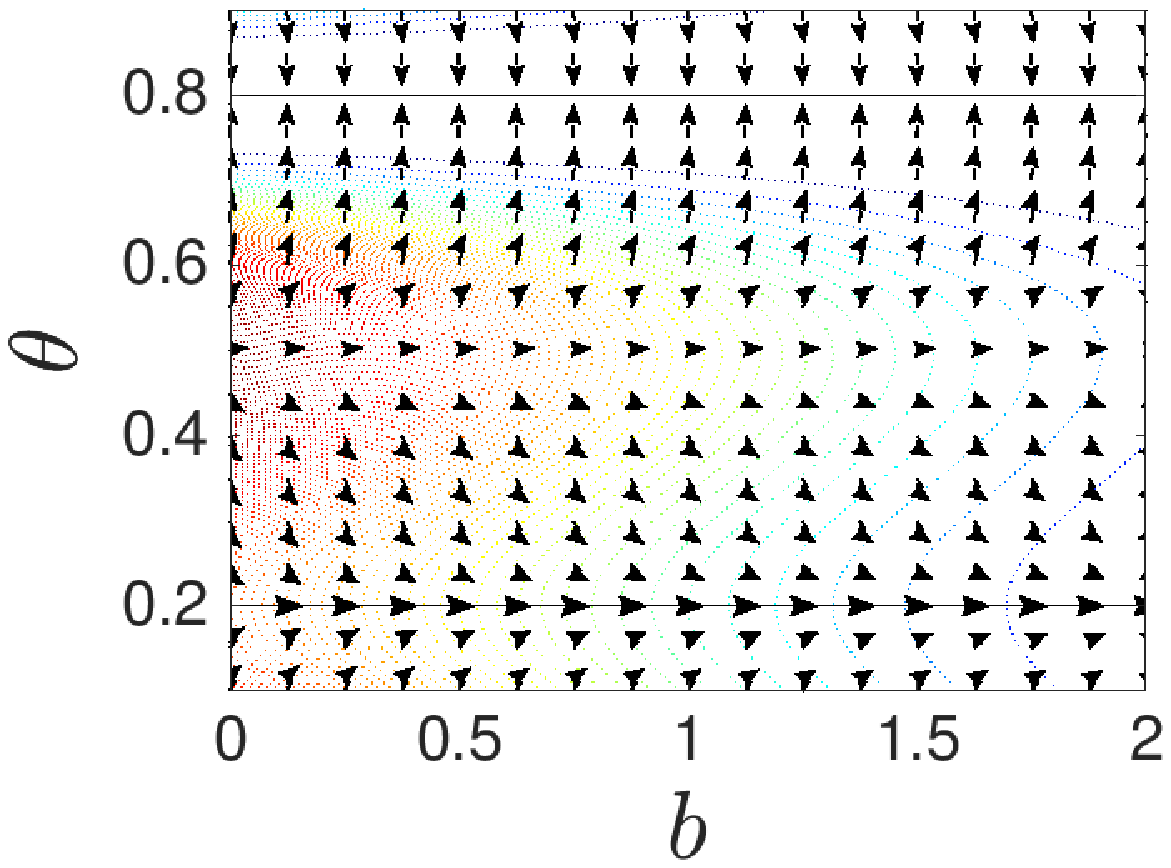}
  \caption{negative gradient directions of $\tilde L$}
\end{subfigure} 
\caption{Illustration of the failure mode suggested by Theorem  \ref{thm:limitation}. In sub-figure (a), the original objective function $L$ has a suboptimal local minimum near $\theta=0.2$ and global minimum near $\theta = 0.8$. In sub-figures (b) and (c), it can be observed that even with the modified objective function $\tilde L$, if $\theta$ is initially near the suboptimal local minimum $(0.2)$, a pure gradient-based local optimization algorithm can still converge to the suboptimal local minimum as $\theta \rightarrow 0.2$ and can diverge in $b$ as $b \rightarrow \infty$. In sub-figure (c), the arrows represent the negative normalized gradient vectors at each point. In sub-figures (b) and (c), the  function $\tilde L$ is plotted along the coordinates $(\theta,b)$ by setting other parameters to be  solutions $(a^*, W^{*})$ of each objective, $\mini_{a,W} \tilde L|_{\theta,b}(a,W)=\tilde L(\theta,a,b,W)$, at each given point $(\theta,b)$. } \vspace{-3pt}
\label{fig:limitation}
\end{figure*}

Figures \ref{fig:limitation} illustrates the novel failure mode proven by Theorem  \ref{thm:limitation}. Note that there exist local minima of $\tilde L$ near $\theta=0.8$ in bounded (and open) subspaces that  are also global minima. Although  there exist local minima of $\tilde L$ in bounded subspaces, there is no local minimum of $L|_{\theta}(a,b,W)$ with a fixed suboptimal $\theta$. The setting used for plotting Figure \ref{fig:limitation} is summarized in Example \ref{example:plot:square}, where the dataset consists of only one sample $(x_1,y_1)$.

\begin{example} \label{example:plot:square}
Let $m=1$, $d_y=1$,  $x_1=0$, and $y_1=-1$. In addition, let $L(\theta)=\ell(f(x_{1};\theta),y_{1})=(\max(0,1-y_{1}f(x_{1};\theta))^3 $. Let $f(x_{1};\theta)=5(- 0.3 e^{-16*(\theta-0.2)^2} - 0.7  e^{-32*(\theta  -0.8)^2} + 0.5)$ for a simple illustration. Because $x_1=0$, we can think of this function as a model with an extra parameter $\theta'$, the effect of which disappears as $\theta' x_1=0$ (e.g., $f(x_{1};\theta)=\bar f(x_{1};\theta,\theta')=5(- 0.3 e^{-16*(\theta'x +\theta-0.2)^2} - 0.7  e^{-32*(\theta'x+\theta  -0.8)^2} + 0.5)$).   
\end{example}

A classical proof using the Weierstrass theorem to guarantee the existence of the optimal solutions in a (nonempty) subspace $S \subseteq\mathbb{R} ^d$ requires a lower semi-continuity of the objective function $\tilde L$ and the existence of a $q \in S$ for which the set $\{q' \in S : \tilde L(q') \le \tilde L(q) \}$ is compact (e.g., see \citealt{bertsekas1999nonlinear} for more discussion on the existence of optimal solutions). In the above example, for the function  $\tilde L|_{\theta}(a,b,W)$ with a fixed suboptimal $\theta$,  the former condition of lower semi-continuity is satisfied, whereas the latter condition of compactness is not.

While Figures \ref{fig:limitation} provides one of the `bad-case' examples, Appendix \ref{sec:numerical_best} and the following section provide some of the `good-case' examples where using $\tilde L$ instead of $L$ helps  optimization of $L$.

\subsection{Analytical examples} \label{sec:analytical_example}

To further understand the properties of eliminating suboptimal local minima in an analytical manner, 
this section presents several analytical  examples. Example \ref{example:0} shows a general case where using $\tilde L$ instead of $L$ helps a gradient-based  optimization method. For the simple illustration of the failure mode, Example \ref{example:1} uses a single data point and squared loss. Example \ref{example:3} is the version of Examples \ref{example:1} with two data points and shows that the value of $\tilde L$ can also approach a suboptimal value. Finally, Example \ref{example:3-2} illustrates the existence of a local minimum of $\tilde L$ via only the existence of a standard global minimum $\theta$ of $L$. In Appendix \ref{sec:app:illustration}, Examples \ref{example:2} and  \ref{example:5} show the same phenomena as those in Examples \ref{example:1} and  \ref{example:3} with a smoothed hinge loss.  

\begin{example} \label{example:0}
Let $\small A[\theta]=\frac{1}{m}[(\frac{\partial f(x_{1};\theta)}{\partial\theta})\T,  \allowbreak \dots,  (\frac{\partial f(x_m;\theta)}{\partial\theta}\allowbreak )\T]\in \RR^{d_\theta \times (md_y)}$ be a matrix, and $r[\varphi]=[\nabla \ell_{y_{1}}(\varphi(x_1))\T \allowbreak , \dots, \nabla \ell_{y_{m}}(\varphi(x_{m}))\T]\T \in\RR^{md_y}$ be a column vector given a function $\varphi: \RR^{d_x} \rightarrow \RR^{d_y}$. The modified objective $\tilde L$ helps a gradient-based method by creating extra decreasing directions as  $r[f(\cdot;\theta)+g(\cdot;a,b,W)] \notin \Null(A[\theta])$  even when  $r[f(\cdot;\theta)] \in \Null(A[\theta])$. This also helps optimization when $r[f(\cdot;\theta)]$ is approximately in $\Null(A[\theta])$ while $r[f(\cdot;\theta)+g(\cdot;a,b,W)]$ is not.  
\end{example}
\begin{example} \label{example:1}
Let $m=1$ and $d_y=1$. In addition, let $L(\theta)=\ell(f(x_{1};\theta),y_{1})=(f(x_{1};\theta)-y_{1})^2$. Accordingly, $\tilde L(\theta,a,b,W)=(f(x_{1};\theta)+a\exp(w\T x_1+b) - y_{1})^2 +\lambda a^2$. Let $\theta$ be a non-global minimum of $L$ as $f(x_{1};\theta) \neq y_{1}$. In particular, let us first consider the case of $f(x_{1};\theta)=2$ and $y_1=1$. Then, $L(\theta)=1$ and 
$
\tilde L(\theta,a,b,W)
=1+2a\exp(w\T x_1+b)+a^{2}\exp(2w\T x_1+2b)+ \lambda a^2.
$
If $(a,b,W)$ is a local minimum, from the stationary point conditions of $\frac{\tilde L(\theta,a,b,W)}{\partial a}=0$ and $\frac{\tilde L(\theta,a,b,W)}{\partial b}=0$, we must have $a=0$, yielding that 
$
\tilde L(\theta,a,b,W)=1.
$
However, a point with $a=0$ is  not a local minimum (with finite $(b,w)$), because with $a <0$ and $|a|>0$ being sufficiently small, 
$ 
\tilde L(\theta,a,b,W)
 =1-2 |a| \cdot \exp(w\T x_1+b) + |a|^2 (\exp(2w\T x_1+2b)+\lambda)
  <1.
$
Hence, there is no local minimum $(a,b,W)\in \RR^{} \times \RR^{} \times \RR^{d_x }$ of $\tilde L|_{\theta}$. Indeed, if we set $a= -\exp(-1/\epsilon)$ and $b=1/\epsilon- w\T x_1$,
$
\tilde L(\theta,a,b,W)=\lambda\exp(-2/\epsilon)\rightarrow0 
$
as $\epsilon \rightarrow 0$, and hence as $a \rightarrow 0^-$ and $b \rightarrow \infty$. This illustrates the case in which $(a,b)$ does not attain a solution in   $\RR \times \RR$. The identical conclusion holds with the general case of $f(x_{1};\theta) \neq y_{1}$ by following  the same steps of reasoning.  
\end{example}

\begin{example} \label{example:3}
Let $m=2$ and $d_y=1$. In addition, $L(\theta)=(f(x_{1};\theta)-y_{1})^2+(f(x_{2};\theta)-y_{2})^{2}$. Let us consider the case of $f(x_{1};\theta)=f(x_{2};\theta)=0$, $y_1=1$, and $y_2=-1$. Then, $L(\theta)=2$. If $(a,b,W)$ is a local minimum, we must have $a=0$ similarly to Example \ref{example:1}, yielding that $\tilde L(\theta,a,b,W)=2$. On the other hand,

\begin{align} \label{eq:example:3}
& \tilde L(\theta,a,b,W) 
\\ \nonumber  & = 2-2a(\exp(w\T x_1+b)-\exp(w\T x_2+b)) +  \varphi(a^{2}),
\end{align} 
where $ \varphi(a^{2})=a^{2}\exp(2w\T x_1+2b)+a^{2}\exp(2w\T x_2+2b)+\lambda a^{2}$. Note that, with a sufficiently small $|a|>0$, the term $\varphi(a^{2})$ becomes negligible. Let $x_1 \neq x_2$. In this case, our $\theta$ with $f(x_{1};\theta)=f(x_{2};\theta)=0$ is not a global minimum. Then, a point with $a=0$ can be shown to be not a local minimum as follows. If $\exp(w\T x_1+b) >\exp(w\T x_2+b) $, with $a>0$ being sufficiently small, $\tilde L(\theta,a,b,W)<2$. If $\exp(w\T x_1+b) <\exp(w\T x_2+b) $, with $a<0$ and $|a|$ being sufficiently small, $\tilde L(\theta,a,b,W)<2$. If $\exp(w\T x_1+b)=\exp(w\T x_2+b) $, since $x_1 \neq x_2$, we can perturb $w$ with an arbitrarily small magnitude to make $\exp(w\T x_1+b) \neq \exp(w\T x_2+b)$, and hence we can yield the above cases. Thus, a point with $a=0$ is not a local minimum. Therefore, there is no local minimum $(a,b,W)$ of $\tilde L|_{\theta}$. Indeed, since $x_1 \neq x_2$, if we set $a= \exp(-1/\epsilon)$, $b=1/\epsilon-w\T x_1$, and $w=-\frac{1}{\epsilon} (x_2-x_1 )$,
$
\tilde L(\theta,a,b,W) 
=(\exp(-\|x_2-x_1\|^2_2 /\epsilon ))+1)^{2}+\lambda\exp(-2/\epsilon)\rightarrow1, $
as $\epsilon \rightarrow 0$, and hence as $a \rightarrow 0^-$, $b \rightarrow \infty$ and $\|w\|\rightarrow \infty $, illustrating the case in which   $(a,b,W)$ does not attain a solution in   $\RR \times \RR \times \RR^{d_x}$.
\end{example}

\begin{example} \label{example:3-2}     
Consider the exact same example as in Example \ref{example:3}, with the exception that $x_1 = x_2$.\ In this case, a $\theta$ with $f(x_{1};\theta)=f(x_{2};\theta)=0$ is a global minimum unlike in Example \ref{example:3}. A point with $a=0$ is indeed a local minimum of $\tilde L$, which can be seen in Equation \eqref{eq:example:3} where for all $a$, $2-2a(\exp(w\T x_1+b)-\exp(w\T x_1+b)) +  \varphi(a^{2})=2+  \varphi(a^{2})\ge 2$.   
\end{example}

\vspace{-8pt}  
\section{Related work} \label{sec:background}
\vspace{-8pt}

There have been many analyses regarding the optimization  of  deep neural networks with    significant over-parameterization \citep{nguyen2017loss,allen2018convergence,du2018gradient,zou2018stochastic} and model simplification \citep{choromanska2015loss,kawaguchi2016deep,hardt2016identity,bartlett2019gradient,du2019width}. In contrast, this paper studies the problem  in the regime without   significant over-parameterization or model simplification.

Our results  are also  different   when compared to the previous study by \citet{liang2018adding}. First, we have solved the  open problem left in the previous study; i.e., our theoretical results are applicable to    practical settings and neural networks with multiple output units and common loss functions. The results in the previous study are  applicable to a neural network with a single output unit for binary classification with particular smoothed hinge loss functions that are not used in common practice. In particular, the previous results are not applicable to multi-class classification or  regression with any loss criteria, or  binary classification with standard loss criteria (e.g., cross entropy loss and   smoothed hinge loss without  twice differentiability). %In general, a theory inapplicable in common practice is an important first step, and the next equally important step would be to bridge a remaining gap toward  a theory applicable in common practice. In this context, this paper has successfully bridged this gap by providing widely applicable novel theoretical results.                    

Second, we proved and demonstrated the  failure mode of eliminating local minima  as a   key contribution via Theorem \ref{thm:limitation} as well as  numerical and analytical examples. The  failure mode  proven in Theorem \ref{thm:limitation} also holds true for  the results in the previous study.\footnote{ This is because the assumptions of Theorem \ref{thm:limitation} are implied by the assumptions used in the previous study, and the construction of the tuple $(\ell,f, \{(x_i,y_i)\}_{i=1}^m, \theta)$ in the proof also accommodates the  setting in  the previous study.} Indeed, Examples \ref{example:plot:square}, \ref{example:2}, and \ref{example:5} as well as  Figure \ref{fig:limitation} directly  illustrate the failure mode of the  results in the previous paper. The previous study does not discuss any possible failure mode of eliminating local minima and, in fact, states that good neural networks are ``just one neuron away'' from bad neural networks (with suboptimal local optima).  Our Theorem \ref{thm:limitation} together with analytical examples prove that such ``good'' neural networks with an added neuron are still subject to their own failure mode, opening up the need of future research.

Third, this paper has introduced a novel and concise proof based on   the PGB necessary condition as well as a longer but more elementary proof. Our  proof  introduced  new insight into why we can eliminate suboptimal local minima; i.e., the global minimum value of the perturbable gradient basis  of an added network is indeed the global minimum value of $L(\theta )$ (see Section \ref{sec:proof_new}  for more details). Beyond the elimination of local minima and the particular modification $\tilde L$, the PGB condition can be used to   analyze other models and modifications, and  it might be helpful to design  new modifications of the objective functions. 

In addition to  the use of the PGB necessary condition,  our proofs also  differ from the previous proofs because the scope and the assumptions of the results are   different. Indeed, the analyses of  one dimensional output with  $y \in \{-1, +1\}$ (the previous study) and  high-dimensional output with  $y \in \RR^{d_y}$ (this paper) are  naturally different. For example, when the matrix $f(X; \theta)=[f(x_1; \theta), ..., f(x_m;\theta)]$ is rank-deficient, all outputs must be simply zero as $f(X; \theta)=0$ in the previous study, whereas $f(X; \theta)$ can be any one of   infinitely  many non-zero (rank-deficient) matrices  in this paper. Unlike the previous study, our proofs also cannot invoke  properties of discrete points and second-order Taylor expansions  because we do not assume  $y \in \{-1, +1\}$ (together with  the particular smoothed hinge loss)  and twice differentiability.

\vspace{-8pt}  
\section{Conclusion}
\vspace{-8pt}

In this paper, we proved that if an algorithm   finds a local minimum of a modified objective function $\tilde L$, then it immediately recovers a global minimum of the original objective function $L$ of an arbitrary deep neural network. However, Theorem \ref{thm:limitation} together with analytical examples  showed that if an algorithm  simply follows negative gradient directions toward a local minimum of $\tilde L$, either it moves toward a global minimum $\theta$ of $L$  or the norm of $(a,b,W)$ approaches infinity. This suggested a hybrid approach of local and global optimization algorithms,  with a mechanism to monitor the norm of $(a,b,W)$. 

From a theoretical viewpoint, we have shown a  reduction of  the problem of getting stuck around an arbitrarily poor local minimum into the  detectable problem of the divergence of the norm. This proven reduction might be useful as a future proof technique in a theoretical literature and as a foundation of a future algorithm in practice. 

In summary, this paper has advanced theoretical understanding of the  properties of the optimization landscape in the  regime that has  not been studied well by previous research with significant over-parameterization or model simplification.
Beyond the elimination of local minima, this paper has introduced the proof technique based the PGB necessary condition of local minima that can be used to study   general  machine learning models and transformations of objective functions.

\vspace{-4pt}
\subsubsection*{Acknowledgements}
\vspace{-5pt}
{ 
We gratefully acknowledge support from NSF grants 1523767 and 1723381, AFOSR grant FA9550-17-1-0165, ONR grant N00014-18-1-2847, Honda Research, and the MIT-Sensetime Alliance on AI. 
}

{%\small
\bibliography{all}
}

\clearpage
        
\appendix

\onecolumn

\begin{center}
\textbf{\LARGE
A Necessary Condition and Elimination of Local Minima for Deep Neural Networks \vspace{10pt}\\ Appendix \vspace{8pt}
}
\end{center}

%\allowdisplaybreaks[0]

\section{Proofs of Theorem \ref{thm:pgb}  }

Let $z\in \RR^{d_z} $ be an arbitrary local minimum of $Q$. From the convexity and differentiability of   $Q_{i}$ and $R_{i}$, we have that
\begin{align} \label{eq:proof:pgb_1}
 & \frac{1}{m}\sum_{i=1}^m  Q_{i}(\phi_{i}^z(\alpha, \epsilon,S))+R_{i}(\varphi_i^{z}(\alpha, \epsilon,S))
\\ \nonumber & \ge\frac{1}{m} \sum_{i=1}^m Q_{i}(\phi_i(z))+R_{i}(\varphi_i(z))+  \partial Q_{i}(\phi_i(z))(\phi_{i}^z(\alpha, \epsilon,S)-\phi_i(z))+  \partial R_{i}(\varphi_i(z))(\varphi_{i}^z(\alpha, \epsilon,S)-\varphi_i(z))
\\ \nonumber & =Q(z)+\frac{1}{m}  \sum_{i=1}^m \partial Q_{i}(\phi_i(z))\phi_{i}^z(\alpha, \epsilon,S)+\partial R_{i}(\varphi_i(z))\varphi_{i}^z(\alpha, \epsilon,S)-  \partial Q_{i}(\phi_i(z))\phi_i(z)-  \partial R_{i}(\varphi_i(z))\varphi_i(z). \end{align}
Since  $z$ is  a local minimum of $Q$, by  the definition of a local minimum, there exists $\epsilon_1>0$ such that $Q(z) \le Q(z')$ for all $z'\in B(z,\epsilon_1)$. Then, for  any $\epsilon \in [0, \epsilon_{1}/2)$ and any $\nu\in\Vcal[z,\epsilon]$, the vector $(z+\epsilon v)$ is also a local minimum because   
$$
 Q(z+\epsilon v)=Q(z)\le Q(z'),
$$
for all $z' \in B(z+\epsilon v,\epsilon_{1}/2) \subseteq  B(z,\epsilon_1)$, where the set inclusion follows from the triangle inequality. This satisfies the definition of a local minimum for $(z+\epsilon v)$. 
Since the composition and the sums of differentiable functions are differentiable,  the vector $(z+\epsilon v)$ is  a differentiable local minimum. Therefore, from the first-order necessary condition of differentiable local minima, there exists $\epsilon_{0} > 0$ such that for any $\epsilon \in [0, \epsilon_{0})$, any $v \in\Vcal[\theta,\epsilon]$, and any $k \in \{1,\dots,d_\theta\}$, 
\begin{align} \label{eq:proof:pgb_2}
\partial_{k} Q(z+\epsilon v) =\frac{1}{m}\sum_{i=1}^m     \partial Q_{i}(\phi_i(\theta))\partial_{k} \phi_i(z+\epsilon v)+\partial R_{i}(\varphi_i(\theta))\partial_{k} \varphi_i(z+\epsilon v)=0,
\end{align}
where we used the fact that $\phi_i(z)=\phi_i(z+\epsilon v)$ and $\varphi_i(z)=\varphi_i(z+\epsilon v)$ for any $v \in \Vcal[z,\epsilon]$. From  \eqref{eq:proof:pgb_2}, there exists $\epsilon_{0} > 0$ such that for any $\epsilon \in [0, \epsilon_{0})$,  any $   S \finsubseteq \Vcal[\theta,\epsilon]$ and any $\alpha \in \RR^{d_\theta\times |S|}$,
\begin{align} \label{eq:proof:pgb_3}
&\frac{1}{m}  \sum_{i=1}^m \partial Q_{i}(\phi_i(z))\phi_{i}^z(\alpha, \epsilon,S)+\partial R_{i}(\varphi_i(z))\varphi_{i}^z(\alpha, \epsilon,S)
\\ \nonumber &=\sum_{k=1}^{d_z} \sum_{j=1}^{|S|} \alpha_{k,j} \left(\frac{1}{m}  \sum_{i=1}^m \partial Q_{i}(\phi_i(z))\partial _{k}\phi_{i}(z+\epsilon S_{j})+\partial R_{i}(\varphi_i(z))\partial _{k}\varphi_i(z+\epsilon S_{j}) \right) 
\\ \nonumber & = 0
\end{align}

where the second line follows the definition of $\phi_{i}^z(\alpha, \epsilon,S)$ and $\varphi_{i}^z(\alpha, \epsilon,S)$, and  the last line follows \eqref{eq:proof:pgb_2}. 

Furthermore, \begin{align} \label{eq:proof:pgb_4}
&\frac{1}{m}  \sum_{i=1}^m \partial Q_{i}(\phi_i(z))\phi_i(z)+  \partial R_{i}(\varphi_i(z))\varphi_i(z)\pm  (1/\rho) \partial R_{i}(\varphi_i(z))\varphi_i(z)
\\ \nonumber &=\sum_{k=1}^{d_z}h^{}(z)_k \left(\frac{1}{m}  \sum_{i=1}^m \partial Q_{i}(\phi_i(z))\partial _{k}\phi_{i}(z)+\partial R_{i}(\varphi_i(z))\partial _{k}\varphi_i(z) \right) +(1-1/\rho)\frac{1}{m}  \sum_{i=1}^m \partial R_{i}(\varphi_i(z))\varphi_i(z)
\\ \nonumber & = (1-1/\rho)\frac{1}{m}  \sum_{i=1}^m \partial R_{i}(\varphi_i(z))\varphi_i(z)
\end{align}

where the second line follows the assumption of the existence of a function $h$ for writing $\phi_i(z)$ and $\varphi_i(z)$, and  the last line follows \eqref{eq:proof:pgb_2}. 

Substituting \eqref{eq:proof:pgb_3} and \eqref{eq:proof:pgb_4} into \eqref{eq:proof:pgb_1},  there exists $\epsilon_{0} > 0$ such that for any $\epsilon \in [0, \epsilon_{0})$,  any $   S \finsubseteq \Vcal[\theta,\epsilon]$ and
any $\alpha \in \RR^{d_\theta\times |S|}$,
\begin{align*}
 & \frac{1}{m}\sum_{i=1}^m  Q_{i}(\phi_{i}^z(\alpha, \epsilon,S))+R_{i}(\varphi_i^{z}(\alpha, \epsilon,S)) \ge Q(z)-(1-1/\rho)\frac{1}{m}  \sum_{i=1}^m \partial R_{i}(\varphi_i(z))\varphi_i(z).
\end{align*}
 This proves the main statement of the theorem. In the case of $\rho=1$,  this shows that on the one hand, there exists $\epsilon_{0} > 0$ such that for any $\epsilon \in [0, \epsilon_{0})$, $Q(z)\le \inf \{\frac{1}{m}\sum_{i=1}^m  Q_{i}(\phi_{i}^z(\alpha, \epsilon,S))+R_{i}(\varphi_i^{z}(\alpha, \epsilon,S)) :S \finsubseteq \Vcal[z,\epsilon] ,\alpha \in \RR^{d_z \times |S|}\}$.
On the other hand, since $\phi_{i}(z)=\sum_{k=1}^{d_z} h^{}(z)_k \partial_{k} \phi_{i} (z)$ and $\varphi_{i}(z)=\rho\sum_{k=1}^{d_z} h^{}(z)_k \partial_{k} \varphi_{i} (z)$ with $\rho=1$,
 we have that $Q(z)\ge \inf \{\frac{1}{m}\sum_{i=1}^m  Q_{i}(\phi_{i}^z(\alpha, \epsilon,S))+R_{i}(\varphi_i^{z}(\alpha, \epsilon,S)) :S \finsubseteq \Vcal[z,\epsilon] ,\alpha \in \RR^{d_z \times |S|}\}$. Combining these  yields the desires statement for  the equality in the case of $\rho=1$.

\qed

\section{Proofs of eliminating local minima} \label{sec:proof}
  A high level idea behind the proofs of Theorems \ref{thm:main} and \ref{thm:realizable_data} in this section (instead of the proof via the PGB necessary condition) follows the idea utilized  by \cite{kawaguchi2016deep} for  deep linear networks.  That is, we  first obtain   possible candidate local minima $\tilde \theta$ via  the first-order necessary condition (i.e., $\{(\theta, a,b,W):a=0\}$), and then consider small perturbations of those candidate local minima. From the definition of local minima, the value at a possible local minimum $\tilde \theta$ must be less than or equal to  the value at any sufficiently small perturbations of the given local minimum $\tilde \theta$. This condition imposes strong constraints on those  candidate local minima, and turns out to be sufficient to prove the desired result with appropriate perturbations and rearrangements, together with the interpolation result with polynomial or simply  based on linear algebra (i.e., we can interpolate $m'$ points via polynomial as the corresponding matrix has rank $m'$). 

In all the proofs of Theorems \ref{thm:main} and \ref{thm:realizable_data} (including the proof with the PGB necessary condition), we let $\theta$ be arbitrary so that we can prove the failure mode  of eliminating the suboptimal local minima in the next section  (Theorem \ref{thm:limitation}) by reusing these proofs. Let $\ell_{y}(q)=\ell(q,y_{})$, and let $\nabla \ell_{y}(\varphi(q))=(\nabla \ell_{y})(\varphi(q))$ be the gradient $\nabla \ell_{y}$ evaluated at an output $\varphi(q)$ of a function $\varphi$.

\subsection{Proof of Theorem \ref{thm:main} without the PGB necessary condition}  

\begin{proof}[Proof of Theorem \ref{thm:main}]
 
Let $\theta$ be fixed. Let $(a,b,W)$ be a local minimum of $\tilde L|_{\theta}(a,b,W):=\tilde L(\theta,a,b,W)$. 
Let $\tilde L|_{(\theta,W)}(a,b)=\tilde L(\theta,a,b,W)$. Since $\ell_{y}:q  \mapsto\ell(q, y)$ is assumed to be differentiable, $\tilde L|_{(\theta,W)}$ is also differentiable (since a sum of differentiable functions is differentiable, and a composition of differentiable functions is differentiable). From the definition of a stationary point of a differentiable function $\tilde L|_{(\theta,W)}$, for all $k \in \{1,2,\dots,d_y\}$, $a_k \frac{\partial \tilde L(\theta,a,b,W)}{\partial a_{k}}= \frac{1}{m}\sum_{i=1}^m (\nabla \ell_{y_i}( f(x_{i};\theta) + g(x_{i};a,b,W)))_{k}a_k\exp(w_k\T x +b_k)+2\lambda a_{k}=\frac{\partial \tilde L(\theta,a,b,W)}{\partial b_{k}}+2\lambda a_{k}^2=2\lambda a_{k}^2=0$, which implies that $a_{k}=0$ for all $k \in \{1,2,\dots,d_y\}$, since $ 2\lambda \neq 0$. Therefore, we have that  
\begin{align} \label{lemma:a=0}
a=0.
\end{align}

This yields $g(x;a,b,W)=0$, and
$$
\tilde L(\theta,a,b,W)=L(\theta).
$$
We now consider perturbations of a local minimum $(a,b,W)$ of $L|_{\theta}$ with $a=0$. Note that, among other equivalent definitions, a function $h:\RR^d \rightarrow \RR$ is said to be differentiable at $q \in \RR^d$ if there exist a vector $\nabla h(q)$ and a function $\varphi (q;\cdot)$ (with its domain being a deleted neighborhood of the origin $0 \in \RR^{d}$) such that  $\lim_{\Delta q \rightarrow0} \varphi_{}(q;\Delta q)=0$, and 
$$
h(q+\Delta q) =h(q)+\nabla h(q)\T\Delta q +\|\Delta q\|\varphi_{}(q;\Delta q),
$$
for any non-zero vector $\Delta q\in \RR^{d}$ that is sufficiently close to $0 \in \RR^{d}$ (e.g., see fundamental increment lemma and the definition of differentiability for multivariable functions). Thus, with sufficiently small perturbations $\Delta a \in \RR^{d_y}$ and $\Delta W =\begin{bmatrix}\Delta w_1 & \Delta w_2 & \dots & \Delta w_{d_{y}}
\end{bmatrix} \in \RR^{d_x \times d_y}$,
there exists a function $\varphi$ such that 

\begin{align*}
& \tilde L(\theta,a+\Delta a ,b,W+ \Delta W)
\\ & =\frac{1}{m} \sum_{i=1}^m \ell_{y_i}( f(x_{i};\theta) + \Delta g_i) +\lambda \|\Delta a\|^2_2
\\ & =\frac{1}{m} \sum_{i=1}^m  \ell_{y_i}( f(x_{i};\theta))+\nabla \ell_{y_i}( f(x_{i};\theta)) \T \Delta g_i
 + \|\Delta g_i\|_{2}\varphi( f(x_{i};\theta);\Delta g_i) + \lambda \|\Delta a\|^2_2,  
\end{align*}
where $\lim_{\Delta q \rightarrow0} \varphi_{}( f(x_{i};\theta);\Delta q)=0$ and $\Delta g_i =g(x_{i};\Delta a,b,W+ \Delta W))$. Here, the last line follows the definition of the differentiability of $\ell_{y_i}$, since $g(x_{i};\Delta a,b,W+ \Delta W)_{k}=\Delta a_{k} \exp(w_k\T x_{i} +\Delta w_k \T x_{i}+b_k)$ is arbitrarily small with sufficiently small $\Delta a_{k}$ and $\Delta w_k$. 

Combining the above two equations, since $(a,b,W)$ is a local minimum, we have that, for any sufficiently small $\Delta a$ and $\Delta w$, 
\begin{align*}
& \tilde L(\theta,a+\Delta a ,b,W+ \Delta W)- \tilde L(\theta,a,b,W) 
\\ & =\frac{1}{m} \sum_{i=1}^m  \nabla \ell_{y_i}( f(x_{i};\theta)) \T \Delta g_i 
 + \frac{1}{m} \sum_{i=1}^m  \|\Delta g_i\|_2\varphi( f(x_{i};\theta);\Delta g_i) + \lambda \|\Delta a\|^2_2
\\ & \ge 0. 
\end{align*}
Rearranging with $\Delta a_{}= \epsilon v$ such that $\epsilon>0$ and $\|v\|_{2} =1$,
and with $\Delta \tilde g_i =g(x_{i}; v,b,W+ \Delta W)$,
\begin{align*}
& \frac{\epsilon}{m} \sum_{i=1}^m  \nabla \ell_{y_i}( f(x_{i};\theta)) \T \Delta \tilde g_i 
\ge - \frac{\epsilon}{m} \sum_{i=1}^m \|\Delta \tilde g_i \|_{2}\varphi( f(x_{i};\theta);\epsilon\Delta \tilde g_i ) - \lambda\epsilon^{2} \|v\|^2_2,  
\end{align*}
since $\Delta g_i=\epsilon\Delta \tilde g_i $.
With $\epsilon>0$, this implies that  
\begin{align*}
& \frac{1}{m} \sum_{i=1}^m  \nabla \ell_{y_i}( f(x_{i};\theta)) \T \Delta \tilde g_i 
\ge -\frac{1}{m} \sum_{i=1}^m \|\Delta \tilde g_i \|_{2}\varphi( f(x_{i};\theta);\epsilon\Delta \tilde g_i ) - \lambda\epsilon^{} \|v\|^2_2.  
\end{align*}
Since $\varphi( f(x_{i};\theta);\epsilon\Delta \tilde g_i)\rightarrow0$ and $\lambda\epsilon^{} \|v\|^2_2\rightarrow 0$ as $\epsilon\rightarrow 0$ ($\epsilon \neq 0$), 
$$
 \sum_{i=1}^m  \nabla \ell_{y_i}(f(x_{i};\theta)) \T g(x_{i}; v,b,W+ \Delta W) \ge 0.
$$
For any $k \in \{1,2,\dots,d_y\}$, by setting $v_{k'}=0$ for all $k' \neq k$,  we have that 
$$
v_{k}\sum_{i=1}^m  (\nabla \ell_{y_i}(f(x_{i};\theta)))_k \exp(w_k\T x_{i} +\Delta w_k \T x_{i}+b_k) \ge 0,
$$
for any $v_k \in \RR$ such that $|v_k|=1$. With $v_k\in \{-\-1,1\},$
$$
 \sum_{i=1}^m  (\nabla \ell_{y_i}(f(x_{i};\theta)))_k \exp(w_k\T x_{i} +b_k)\exp(\Delta w_k \T x_{i}) = 0.
$$
By setting $\Delta w_k = \bar \epsilon_k u_k$ such that $\bar \epsilon_{k}>0$ and $\|u\|_{2} =1$,
\begin{align*} 
&\sum_{t=0}^\infty \frac{\bar \epsilon_k^t }{t!} \sum_{i=1}^m  (\nabla \ell_{y_i}(f(x_{i};\theta)))_k \exp(w_k\T x_{i} +b_k)(u_k\T x_{i})^{t}  
  =0, 
\end{align*}  
since $\exp(q) = \lim_{T \rightarrow \infty} \sum_{t=0}^T \frac{q^t}{t!}$  and a finite sum of limits of convergent sequences is the limit of the finite sum. Rewriting this using $z_t=\sum_{i=1}^m (\nabla \ell_{y_i}(f(x_{i};\theta)))_k \exp(w_k\T x_{i} +b_k)(u_k\T x_{i})^{t}$,
\begin{equation} \label{eq:induction1}
\lim_{T \rightarrow \infty}\sum_{t=0}^T \frac{\bar \epsilon_k^t }{t!} z_t=0.
\end{equation} 
We now show that $z_p=0$ for all $p \in \mathbb{N}_0$ by induction. Consider the base case with $p=0$. Equation \eqref{eq:induction1} implies that 
$$
\lim_{T \rightarrow \infty} \left(z_0 +\sum_{t=1}^T \frac{\bar \epsilon_k^t }{t!} z_t \right) =z_0 + \lim_{T \rightarrow \infty} \sum_{t=1}^T \frac{\bar \epsilon_k^t }{t!} z_t  = 0
$$ 
since $\lim_{T \rightarrow \infty} \sum_{t=1}^T \frac{\bar \epsilon_k^t }{t!} z_t $ exists (which follows that $\lim_{T \rightarrow \infty}\sum_{t=0}^T \frac{\bar \epsilon_k^t}{t!} z_t=0$ exists).
Here, $ \lim_{T \rightarrow \infty} \sum_{t=1}^T \frac{\bar \epsilon_k^t }{t!} z_t \rightarrow0$ as $\bar \epsilon \rightarrow 0$, and hence $z_0=0$. Consider the inductive step with the inductive hypothesis that $z_{t}=0$ for all $t\le p - 1$. Similarly to the base case, Equation \eqref{eq:induction1} implies $$
\sum_{t=0}^{p-1} \frac{\bar \epsilon_k^t }{t!} z_t+\frac{\bar \epsilon_k^p }{p!}z_p +\lim_{T \rightarrow \infty} \sum_{t=p+1}^T \frac{\bar \epsilon_k^t }{t!} z_t=0.
$$
Multiplying $p!/\bar \epsilon_k^p$ on both sides,
since $\sum_{t=0}^{p-1} \frac{\bar \epsilon_k^t }{t!} z_t=0$ from the inductive hypothesis,$$
z_p +\lim_{T \rightarrow \infty} \sum_{t=p+1}^T \frac{\bar \epsilon_k^{t-p} p!}{t!} z_t=0.
$$
Since $\lim_{T \rightarrow \infty}\sum_{t=p+1}^T \frac{\bar \epsilon_k^{t-p} p!}{t!} z_t \rightarrow0$ as $\bar \epsilon \rightarrow 0$, we have that  $z_p=0$, which finishes the induction. Therefore, for any $k \in \{1,2,\dots,d_y\}$ and any $p\in \mathbb{N}_0$,
\begin{align} \label{lemma:main_lemma}
\sum_{i=1}^m  (\nabla \ell_{y_i}(f(x_{i};\theta)))_k \exp(w_k\T x_{i} +b_k)(u_k\T x_{i})^{p}=0.
\end{align}   
Let $x \otimes x$ be the tensor product of the vectors $x$ and $x^{\otimes p}= x \otimes \cdots \otimes x$ where $x$ appears $p$ times. For a $p$-th order tensor  $M \in \RR^{d \times \cdots \times d}$ and $p$ vectors $u^{(1)},u^{(2)},\dots,u^{(p)} \in \RR^d$,  defines
$$
M(u_k^{(1)},u_k^{(2)},\dots,u_k^{(p)})=\sum_{1 \le i_1 \cdots i_p \le d} M_{i_{1}\cdots i_{p}} u^{(1)}_{i_1} \cdots u^{(p)}_{i_p}.
$$

Let $\xi_{i,k}=(\nabla \ell_{y_i}( f(x_{i};\theta)))_{k}\exp(w_k\T x_{i} +b_k)$. Then, for any $k \in \{1,2,\dots,d_y\}$ and any $p \in \mathbb{N}_{0}$, 
\begin{align*}
\max_{\substack{u^{(1)},\dots,u^{(p)}:\\ \|u^{(1)}\|_{2}=\cdots=\|u^{(p)}\|_{2}=1}} \left(\sum_{i=1}^m  \xi_{i,k} x_i^{\otimes p} \right)(u^{(1)},\dots,u^{(p)})
& = \max_{u: \|u\|_2=1} \left(  \sum_{i=1}^m  \xi_{i,k} x_i^{\otimes p} \right)(u,u,\dots,u)
\\ &=\max_{u: \|u\|_2=1} \sum_{i=1}^m  \xi_{i,k}(u\T x_i)^p =0.
\end{align*}
\normalsize

where the first line follows theorem 2.1 in \citep{zhang2012best}, and the last line follows Equation \eqref{lemma:main_lemma}. This implies that
\begin{align} \label{eq:all_poly_zero}
& \sum_{i=1}^m (\nabla \ell_{y_i}( f(x_{i};\theta)))_{k}\exp(w_k\T x_{i} +b_k) \vect(x_i^{\otimes p})
 =0\in \RR^{d_{x}^{p}}.
\end{align}
\normalsize

Using Equation \eqref{eq:all_poly_zero}, we now prove statement (i). For any $\theta'$, there exist $p$ and $u_{t,k}$ (for $t=0,\dots,p$ and $k = 1,\dots,d_y$) such that 
\begin{align*}
m(L(\theta')-L(\theta))
 &\ge  \sum_{i=1}^m  \nabla \ell_{y_i}(f(x_i;\theta))\T (f(x_i;\theta')-f(x_i;\theta))
\\ &= \sum_{j=1}^{m'} \sum_{i\in \mathcal{I}_j}  \nabla \ell_{y_i}(f(x_i;\theta))\T (f(x_i;\theta')-f(x_i;\theta))
\\ & =\sum_{j=1}^{m'}  \sum_{k=1}^{d_y} \hspace{-5pt} \underbrace{(f(\bar x_j;\theta')-f(\bar x_j;\theta))_{k}}_{\substack{ \\ \\ = \exp(w_k\T \bar x_{j} +b_k)\sum_{t=0}^p u\T_{t,k} \vect(\bar  x_j^{\otimes t})}} \hspace{-8pt} \sum_{i\in \mathcal{I}_j}  \nabla \ell_{y_i}(f(x_i;\theta))_k 
\\ &= \sum_{t=0}^p   \sum_{k=1}^{d_y} u\T_{t,k} \underbrace{\sum_{i=1}^{m}    \nabla \ell_{y_i}(f(x_i;\theta))_k \exp(w_k\T x_{i} +b_k)\vect(x_i^{\otimes t})}_{\textstyle \small =0 \text{ from Equation \eqref{eq:all_poly_zero}}}
\\ & = 0, 
\end{align*}  
\normalsize 
where the first line follows from the assumption that $\ell_{y_i}$ is convex and differentiable, and the third line follows from the fact that $\bar x_j=x_{}$ for all $x \in \mathcal{I}_j$. 
The forth line follows from the fact that the vector $\vect( x_i^{\otimes t})$ contains all monomials in $x_i$ of degree $t$, and $m'$ input points  $\bar x_1, \dots, \bar x_{m'}$ are distinct, which allows the basic existence (and construction) result of a polynomial interpolation of the finite $m'$ points; i.e.,  with $p$ sufficiently large ($p=m'-1$ is sufficient), for each $k$, there exists $u_{t,k}$ such that $\sum_{t=0}^p u\T_{t,k} \vect(\bar  x_j^{\otimes t})=q_{j,k}$ for any $q_{j,k} \in \RR$ for all $j \in \{1,\dots, m'\}$ (e.g., see equation (1.9) in \citealt{gasca2000polynomial}), in particular, including $q_{j,k}=(f(\bar x_j;\theta')-f(\bar x_j;\theta))_k \exp(-w_k\T \bar x_{j} -b_k)$.    

Therefore, we have that, for any $\theta'$, $L(\theta') \ge L(\theta)$, which proves statement (i). Statement (ii) directly follows from Equation \eqref{lemma:a=0}.
\end{proof}

\subsection{Proof of Theorem \ref{thm:realizable_data}} 
\begin{proof}[Proof of Theorem \ref{thm:realizable_data}]
 Let $\theta$ be fixed. Let $(a,b,W)$ be a local minimum of $\tilde L|_{\theta}(a,b,W):=\tilde L(\theta,a,b,W)$. Then, for any $k \in \{1,2,\dots,d_y\}$, there exist $p$ and $u_{t,k}$ (for $t=0,\dots,p$) such that 
\begin{align*}
 \sum_{i=1}^m(\nabla \ell_{y_i}( f(x_{i};\theta)))_{k}^2
 & = \sum_{j=1}^{m'} | \mathcal{I}_j| (\nabla \ell_{f^{*}(\bar x_j)}( f(\bar x_{j};\theta)))_{k}^2 
\\ & =\sum_{t=0}^p u\T_{t,k} \sum_{i=1}^{m}   (\nabla \ell_{y_{i}}( f(x_{i};\theta)))_{k}\exp(w_k\T x_{i} +b_k) \vect(x_i^{\otimes t})
\\ & = 0,  
\end{align*}
where the first line utilizes Assumption \ref{assump:realizable}. The second line follows from the fact that since $m'$ input points $\bar x_1, \dots, \bar x_{m'}$ are distinct,
with $p$ sufficiently large ($p=m'-1$ is sufficient), for each $k$, there exist $u_{t,k}$ for $t=0,\dots,p$ such that $\sum_{t=0}^p u\T_{t,k} \vect(  x_i^{\otimes t})=(\nabla \ell_{f^{*}(\bar x_j)}( f(\bar x_{j};\theta)))_{k} \exp(-w_k\T  \bar x_{j} -b_k)| \mathcal{I}_j|^{-1}$ (similarly to the proof of Theorem \ref{thm:main}). The third line follows from Equation \eqref{eq:all_poly_zero}. Here, Equation \eqref{eq:all_poly_zero} still holds since it is obtained in the proof of Theorem \ref{thm:main} under only the assumption that the function $\ell_{y_{i}}:q \mapsto\ell(q, y_{i})$ is differentiable for any $i \in \{1,\dots,m\}$, which is still satisfied by Assumption \ref{assump:loss2}.  

This implies that for all $i \in \{1,\dots,m\}$, $\nabla \ell_{y_i}( f(x_{i};\theta))=0$, which proves statement (iii) because of Assumption \ref{assump:loss2}. Statement (i) directly follows from Statement (iii). Statement (ii) directly follows from Equation \eqref{lemma:a=0}. 
\end{proof}

\section{Proof of Theorem \ref{thm:limitation}} \label{sec:app:proof}
The proofs of Theorems \ref{thm:main} and \ref{thm:realizable_data} (including the proof via the PGB necessary condition) are designed such that the proof of Theorem \ref{thm:limitation} is simple, as shown below. Given a  function $\varphi(q) \in \RR^{d}$ and a vector $v \in \RR^{d'}$, let $\frac{\partial\varphi (q) }{\partial v}$ be a $d \times d'$ matrix with each entry  $(\frac{\partial\varphi (q) }{\partial v})_{i,j}=\frac{\partial(\varphi (q))_{i} }{\partial v_{j}}$. 

\begin{proof}[Proof of Theorem \ref{thm:limitation}]
Let Assumption \ref{assump:loss} hold (instead of Assumptions \ref{assump:loss2} and \ref{assump:realizable}). In the both versions of  our proofs of Theorem \ref{thm:main}, $\theta$ was arbitrary and $(a,b,W)$ was an arbitrary local minimum of $\tilde L|_{\theta}(a,b,W):=\tilde L(\theta,a,b,W)$. Thus,
the same proof proves that, for any $\theta$, at every local minimum $(a,b,W)\in \RR^{d_y} \times \RR^{d_y} \times \RR^{d_x \times d_y}$ of $\tilde L|_{\theta}$, $\theta$ is a global minimum of $L$. Thus, based on the logical equivalence ($p\rightarrow q\equiv \lnot q\rightarrow \lnot p$), if $\theta$ is a not global minimum of $L$, then there is no local minimum $(a,b,W)\in \RR^{d_y} \times \RR^{d_y} \times \RR^{d_x \times d_y}$ of $\tilde L|_{\theta}$, proving the first statement in the case of using Assumption \ref{assump:loss}. Instead of Assumption \ref{assump:loss}, if  Assumptions \ref{assump:loss2} and \ref{assump:realizable} hold, then the exact same proof as above (with Theorem \ref{thm:main} being replaced by Theorem \ref{thm:realizable_data}) proves the first statement.

Example \ref{example:plot:square} with the square loss or the smoothed hinge loss suffices to prove the second statement. However, to obtain better theoretical insight, let us consider a more general construction of the desired tuples $(\ell,f, \{(x_i,y_i)\}_{i=1}^m)$ to prove the second statement. Let $\theta \in \RR^{d_\theta}$. In addition, let $A[\theta]=\frac{1}{m}[(\frac{\partial f(x_{1};\theta)}{\partial\theta})\T \ \cdots \ (\frac{\partial f(x_m;\theta)}{\partial\theta})\T]\in \RR^{d_\theta \times (md_y)}$ be a matrix, and $r[\varphi]=[\nabla \ell_{y_{1}}(\varphi(x_1))\T \ \cdots \ \nabla \ell_{y_{m}}(\varphi(x_{m}))\T]\T \in\RR^{md_y}$ be a column vector given a function $\varphi: \RR^{d_x} \rightarrow \RR^{d_y}$. Then, \begin{align*}
\frac{\partial L(\theta)}{\partial \theta} &=\frac{1}{m} \sum_{i=1}^m \nabla \ell_{y_{i}}(f(x_{i};\theta))\T \frac{\partial f(x_{i};\theta)}{\partial\theta}
= (A[\theta]r[f(\cdot;\theta)])\T,
\end{align*}
and
$$
\frac{\partial \tilde L(\theta,a,b,W)}{\partial \theta} =(A[\theta]r[f(\cdot;\theta)+g(\cdot;a,b,W)])\T.
$$

Here, the equality $A[\theta]r[f(\cdot;\theta)]=0$ is equivalent to  $r[f(\cdot;\theta)] \in \Null(A[\theta])$, where $\Null(A[\theta])$ is the null space of the matrix $A[\theta]$. Therefore, any tuple $(\ell,f, \{(x_i,y_i)\}_{i=1}^m)$ such that $r[f(\cdot;\theta)] \in \Null(A[\theta]) \Rightarrow r[f(\cdot;\theta)+g(\cdot;a,b,W)] \in \Null(A[\theta])$ at a suboptimal $\theta$ suffices to provide a  proof for the second statement. An (infinite) set of tuples $(\ell,f, \{(x_i,y_i)\}_{i=1}^m)$ such that there exists a suboptimal $\theta$ of $L$ with $A[\theta]=0$  (e.g., Example \ref{example:plot:square}) satisfies this condition, which proves the second statement. 
\end{proof}

\section{Additional numerical  examples for good cases} \label{sec:numerical_best}

For using $\tilde L$ instead of $L$, we show the failure mode and `bad-case'  scenarios in Section \ref{sec:app:experiment} and Appendix \ref{sec:app:illustration}. Accordingly, to have a balance, this section considers some of `good-case' scenarios where using $\tilde L$ instead of $L$ helps  optimization of $L$. Figure \ref{fig:exp} shows the histograms of training loss values after training with original networks $f$ minimizing $L$,   and modified  networks $\tilde f$ minimizing $\tilde L$ with and without the failure mode detector based on Theorems \ref{thm:main}, \ref{thm:realizable_data} and \ref{thm:limitation}. We used a simple  failure mode detector, which   automatically restarted the optimizer to a random point during training when $\|a\|_{2}+\|b\|_{2}+\|W\|_{2} \ge 7$. The histograms were plotted with the results of 1000 random trials  for Semeion dataset and of 100 random trials for KMNIST dataset, for each method. Semeion \citep{brescia1994semeion} is a dataset of handwritten digits and KMNIST \citep{clanuwat2018deep} is a dataset of Japanese letters. We used the exact same experimental settings for both the original networks $f$ and the modified networks $\tilde f$ with and without the failure mode detector. We used  a standard variant of LeNet \citep{lecun1998gradient} with ReLU activations: two convolutional layers with 64 $5 \times 5$ filters, followed by a fully-connected layer with 1024 output units and the output layer. The AdaGrad optimizer was employed with the mini-batch size of 64. 

\begin{figure*}[h!]
\center
\begin{subfigure}[b]{0.31\textwidth}
  \includegraphics[width=\textwidth, height=0.7\textwidth]{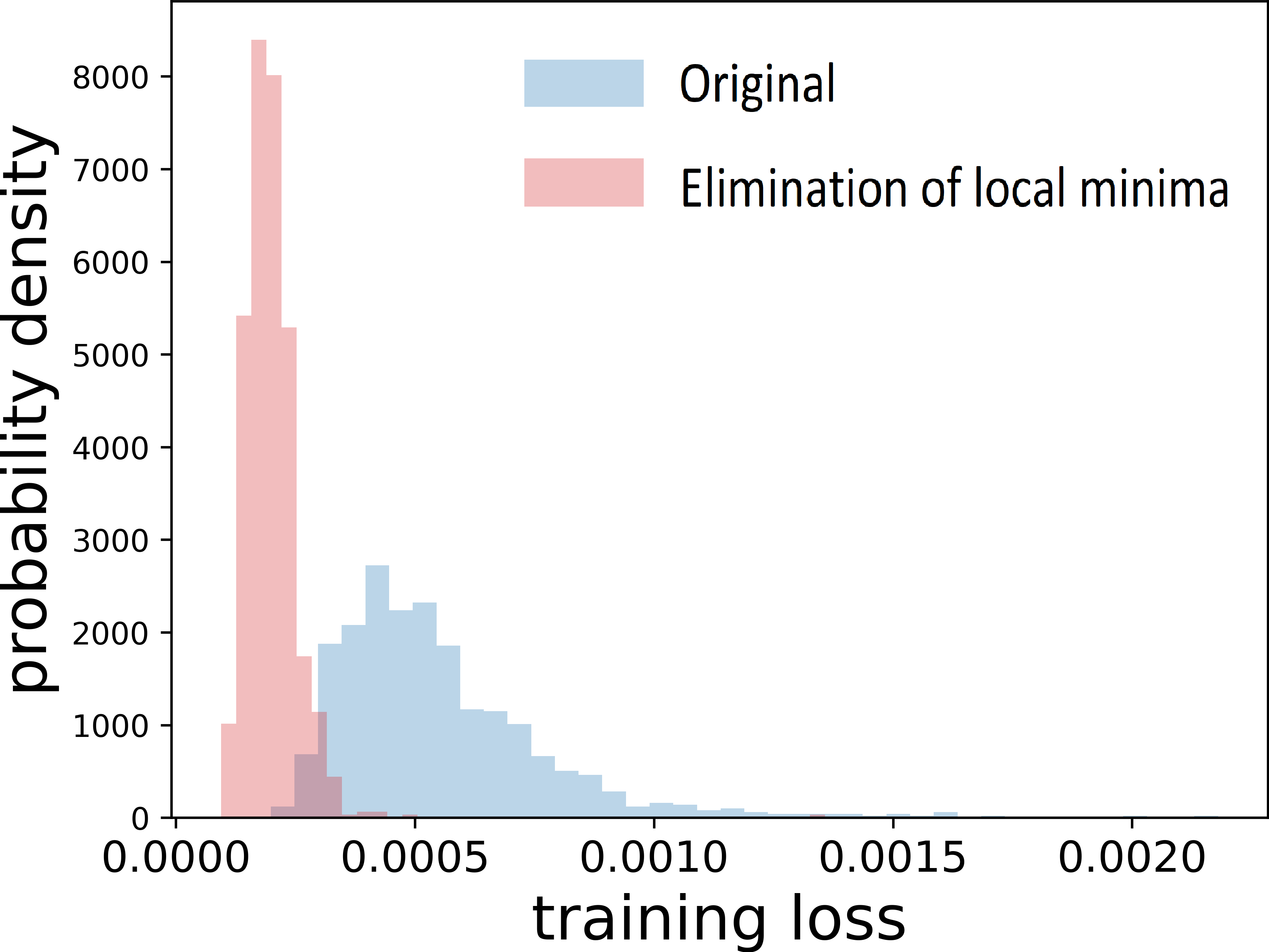}
  \caption{Semeion} 
\end{subfigure} \hspace{30pt}
\begin{subfigure}[b]{0.31\textwidth}
  \includegraphics[width=\textwidth, height=0.7\textwidth]{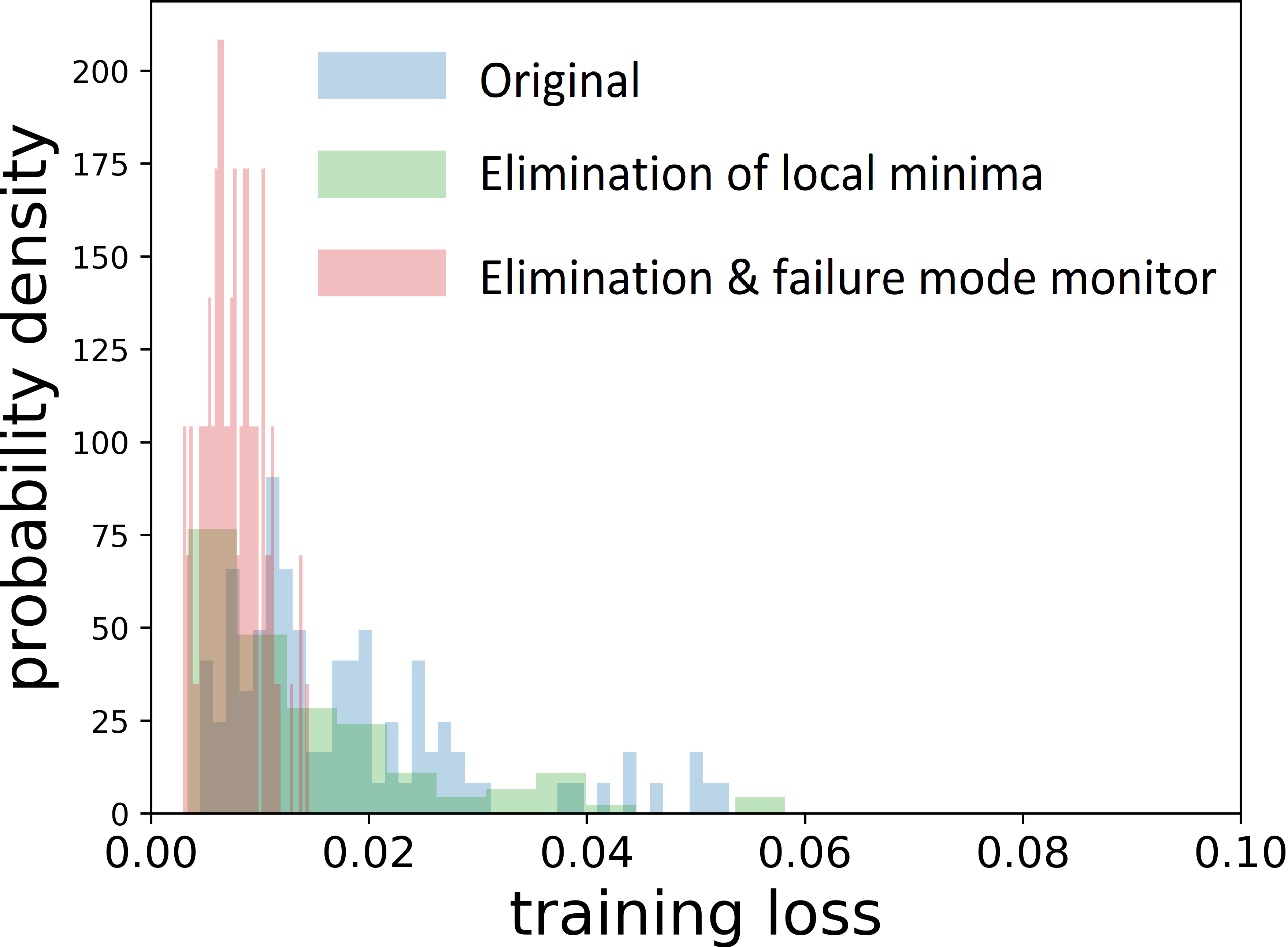}
  \caption{KMNIST}
\end{subfigure} 
\captionof{figure}{Histogram of loss values after training with original networks $f$ minimizing $L$ (original),   modified  networks $\tilde f$ minimizing $\tilde L$  (elimination of local minima), and modified  networks $\tilde f$ minimizing $\tilde L$ with the failure mode detector (elimination \& failure mode monitor). The plotted training loss values are the values of the standard training objective $L$ for both original networks $f$ (minimizing $L$) and modified  networks $\tilde f$ (minimizing $\tilde L$) with and without the failure mode detector. The elimination of local minima helped a gradient-based method for Semeion, and did not help it much for KMNIST. For KMNIST, the novel failure mode of the elimination was detected by monitoring the norms of $(a,b,W)$ to restart and search a better subspace.}
\label{fig:exp}
\end{figure*}

\section{Additional numerical and analytical examples to illustrate the failure mode} \label{sec:app:illustration}

Figure \ref{fig:limitation_hinge} illustrates the novel failure mode proven by Theorem  \ref{thm:limitation}.  The setting used for plotting Figure \ref{fig:limitation_hinge}  is exactly same as that in Figure \ref{fig:limitation}
(i.e., Example \ref{example:plot:square}) except that $\ell(f(x_{1};\theta),y_{1})=(f(x_{1};\theta)-y_{1})^2$ and $y_{1}=f(x_{1};0.8)$.

\begin{figure*}[h!]
\center
\begin{subfigure}[b]{0.3\textwidth}
  \includegraphics[width=\textwidth]{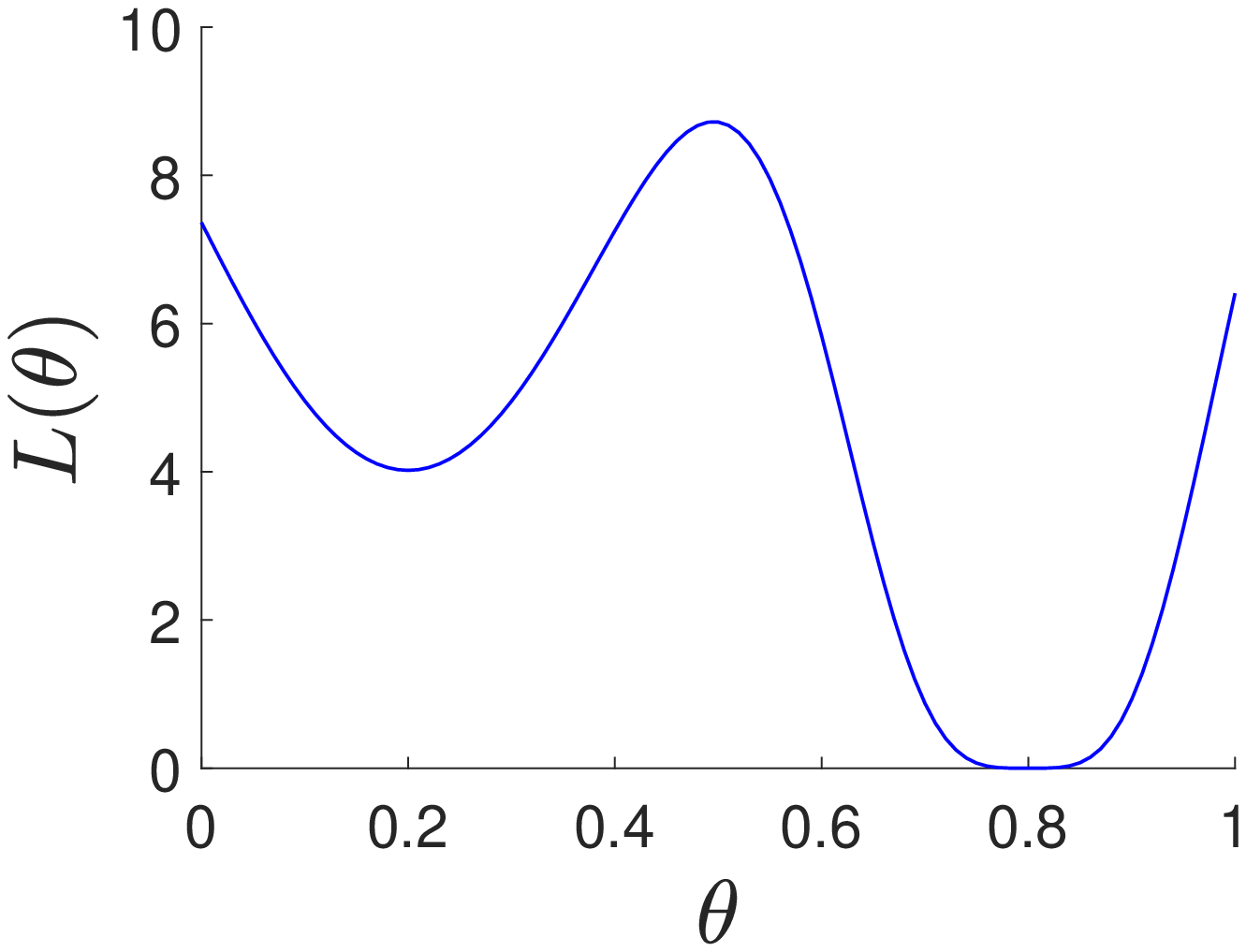} 
  \caption{original objective function $L$} 
\end{subfigure}
\begin{subfigure}[b]{0.33\textwidth}
  \includegraphics[width=\textwidth]{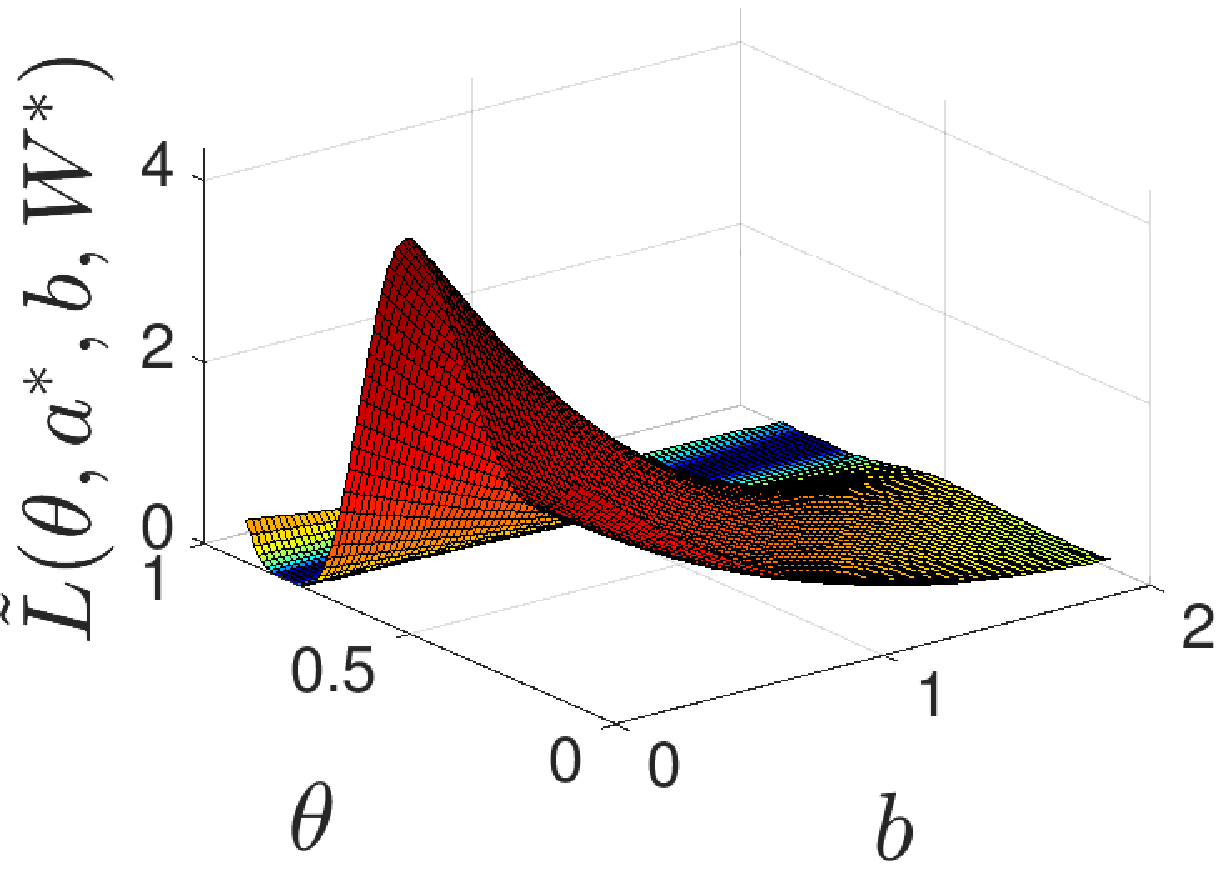}
  \caption{modified objective function  $\tilde L$}
\end{subfigure} 
\begin{subfigure}[b]{0.35\textwidth}
  \includegraphics[width=\textwidth]{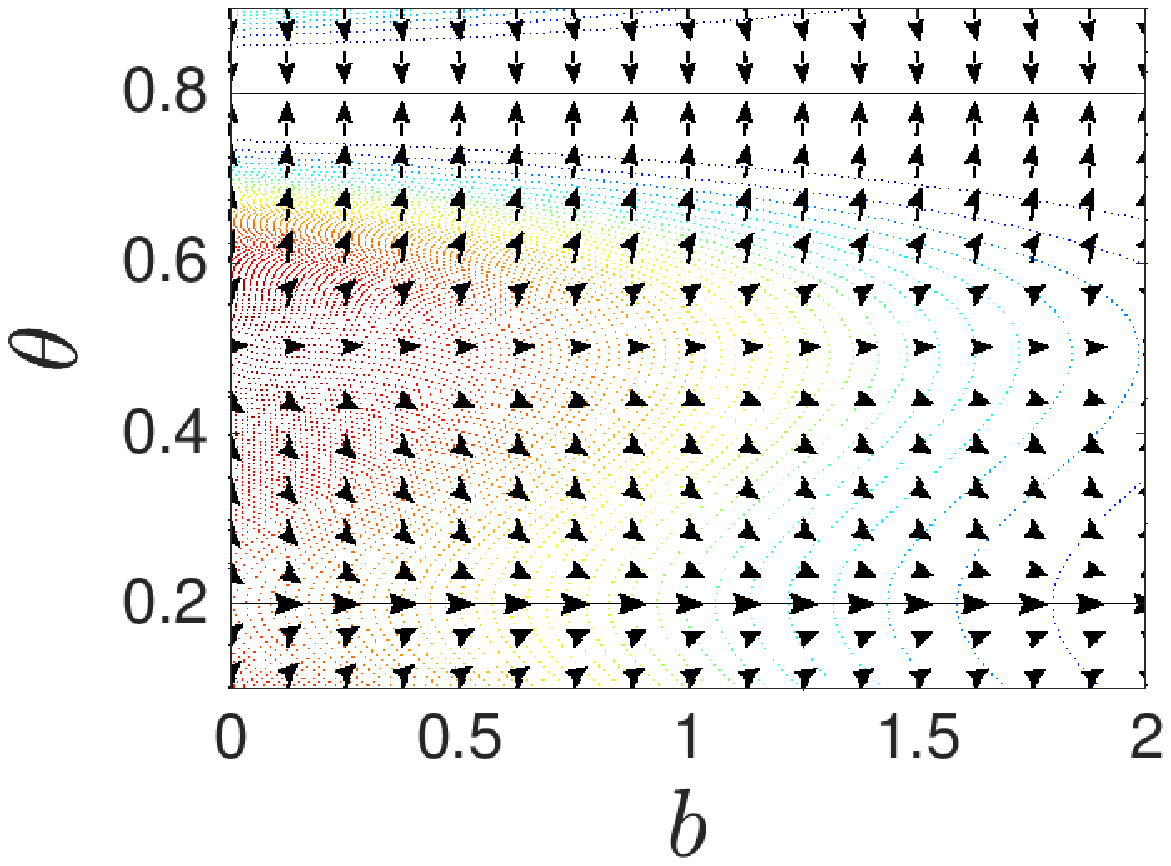}
  \caption{ negative gradient directions of   $\tilde L$}
\end{subfigure} 
\caption{Illustration of the failure mode suggested by Theorem \ref{thm:limitation} with the squared loss. The qualitatively identical behavior as that in Figure \ref{fig:limitation} can be observed.  }
\label{fig:limitation_hinge}
\end{figure*}

Examples \ref{example:2} and  \ref{example:5} illustrate the same phenomena as those in Examples \ref{example:1} and  \ref{example:3} with a smoothed hinge loss instead of the squared loss.

\begin{example} \label{example:2}
Let $m=1$ and $d_y=1$. In addition, $L(\theta)=\ell(f(x_{1};\theta),y_{1})=(\max(0,1-y_{1}f(x_{1};\theta))^3$. Accordingly, $\tilde L(\theta,a,b,W)=(\max(0,1-y_{1}  f(x_{1};\theta)-y_{1}a\exp(w\T x_1+b) )^3+\lambda a^2$. Let $\theta$ be a non-global minimum of $L$ as $f(x_{1};\theta) \neq y_{1}$, in particular, by setting $f(x_{1};\theta)=-1$ and $y_1=1$. Then, $L(\theta)=8$. 
If $(a,b,W)$ is a local minimum, we must have $a=0$ similarly to Example \ref{example:1}, yielding that
$
\tilde L(\theta,a,b,W)=8.
$
However, a point with $a=0$ is  not a local minimum, since with $a >0$ being sufficiently small, 
$$
\tilde L(\theta,a,b,W)=(2-a\exp(w\T x_1+b))^3+\lambda a^2<8.
$$ 
Hence, there is no local minimum $(a,b,W)\in \RR^{} \times \RR^{} \times \RR^{d_x }$ of $\tilde L|_{\theta}$.
Indeed, if we set $a= -2\exp(-1/\epsilon)$ and $b=1/\epsilon- w\T x_1$,
$
\tilde L(\theta,a,b,W)
=\lambda\exp(-2/\epsilon)\rightarrow0 
$
as $\epsilon \rightarrow 0$, and hence as $a \rightarrow 0^-$ and $b \rightarrow \infty$. This illustrates the case in which $(a,b)$ does not attain a solution in   $\RR \times \RR$. The identical conclusion holds with the general case of $f(x_{1};\theta) \neq y_{1}$ by following the same logic. 
\end{example}

\begin{example} \label{example:5}
Let $m=2$ and $d_y=1$. In addition, $L(\theta)=(\max(0,1-y_{1}f(x_{1};\theta))^3+(\max(0,1-y_{2}f(x_{2};\theta))^3$. Moreover, let $x_1 \neq x_2$. Finally, let $f(x_{1};\theta)=-1, f(x_{2};\theta)=1$, $y_1=1$, and $y_2=-1$. If $(a,b,W)$ is a local minimum, we must have $a=0$ similarly to Example \ref{example:1}, yielding $\tilde L(\theta,a,b,W)=16$. However, a point with $a=0$ is not a local minimum, which follows from the perturbations of $(a,W)$ in the same manner as in Example \ref{example:3}. 
Therefore, there is no local minimum $(a,b,W)$ of $\tilde L|_{\theta}$. Indeed, if we set $a=2 \exp(-1/\epsilon)$, $b=1/\epsilon-w\T x_1$, and $w=-\frac{1}{\epsilon} (x_2-x_1 )$,
\begin{align*}
\tilde L(\theta,a,b,W)=(2+2\exp(-\|x_2-x_1\|^2_2 /\epsilon ))^{3}+  \lambda\exp(-2/\epsilon)\rightarrow8
\end{align*}
as $\epsilon \rightarrow 0$, and hence as $a \rightarrow 0^-$, $b \rightarrow \infty$ and $\|w\|\rightarrow \infty $, illustrating the case in which $(a,b,W)$ does not attain a solution in   $\RR \times \RR \times \RR^{d_x}$.
\end{example}

\end{document}